%% file: main.tex
\documentclass[runningheads]{llncs}
\usepackage{graphicx}
\usepackage[dvipsnames,table,xcdraw]{xcolor}
\usepackage{tikz}
\usepackage{comment}
\usepackage{amsmath,amssymb} % define this before the line numbering.
\usepackage{color}
\usepackage{placeins}
\usepackage{cite}

\usepackage[accsupp]{axessibility}  % Improves PDF readability for those with disabilities.

\usepackage{graphicx}

\usepackage{array}
\usepackage{subcaption}
\usepackage{multirow}
\usepackage[flushleft]{threeparttable}
\usepackage{comment}
\usepackage{algorithmic}
\usepackage{algorithm}
\usepackage{booktabs}

\usepackage{amssymb}% http://ctan.org/pkg/amssymb
\usepackage{pifont}% http://ctan.org/pkg/pifont

\usepackage{fancyvrb}

\usepackage{graphicx}
\usepackage{enumitem}
\usepackage{wrapfig}
\usepackage{lipsum}

\usepackage{soul}
\usepackage{amsmath}
\usepackage[bb=dsserif]{mathalpha}
\usepackage{bm}

\usepackage{makecell}

\newcommand{\cmark}{\text{\ding{51}}}%
\newcommand{\xmark}{\text{\ding{55}}}%

\definecolor{LightOrange}{rgb}{1,0.85,0.8}
\definecolor{LightPurple}{rgb}{1.0,0.80,0.95}
\definecolor{LightGreen}{rgb}{0.93,0.98,0.96}
\colorlet{LightGray}{gray!20}

\def\red{\color{Red}}
\def\green{\color{Green}}

\newcommand*\samethanks[1][\value{footnote}]{\footnotemark[#1]}
\newcommand{\around}{{\raise.17ex\hbox{$\scriptstyle\sim$}}}
\newcommand{\tablestyle}[2]{\setlength{\tabcolsep}{#1}\renewcommand{\arraystretch}{#2}\centering\footnotesize}

\newcolumntype{x}[1]{>{\centering\arraybackslash}p{#1pt}}
\newcommand{\app}{\raise.17ex\hbox{$\scriptstyle\sim$}}

\def\x{\times}

\newlength\savewidth\newcommand\shline{\noalign{\global\savewidth\arrayrulewidth
  \global\arrayrulewidth 1pt}\hline\noalign{\global\arrayrulewidth\savewidth}}
  
\makeatletter\renewcommand\paragraph{\@startsection{paragraph}{4}{\z@}
  {.5em \@plus1ex \@minus.2ex}{-.5em}{\normalfont\normalsize\bfseries}}\makeatother

\def\fig#1{Fig.~\ref{fig:#1}}

\newcommand{\tb}[3]{\setlength{\tabcolsep}{#2mm}\begin{tabular}{#1}#3\end{tabular}}

\newtheorem{observation}[theorem]{Observation}

\newcommand{\ol}[3]{\begin{#1}[leftmargin=*,topsep=0pt]\setlength{\itemsep}{#2mm}#3\end{#1}}
\newcommand{\olp}[3]{\begin{#1}[leftmargin=\parindent,topsep=0pt]\setlength{\itemsep}{#2mm}#3\end{#1}}
\newcommand\newsubcap[1]{\phantomcaption%
       \caption*{\figurename~\thefigure(\thesubfigure). #1}}

\def\tablecite#1#{%
  \def\pretablecite{#1}%
  \tableciteaux}
\def\tableciteaux#1{%
  \textsuperscript{\expandafter\originalcite\pretablecite{#1}}%
}

\def\fontsmall#1#{
    \fontsize{8}{12}
    #1
    \selectfont
}

\def\tablefontsmall#1#{
    \fontsize{8}{10}
    #1
    \selectfont
}

\def\tablefont#1#{
    \fontsize{9}{12}
    #1
    \selectfont
}

\let\vec\mathbf

\usepackage[pagebackref,breaklinks,colorlinks]{hyperref}

\input{math_commands.tex}

\begin{document}

\pagestyle{headings}
\mainmatter
\def\ECCVSubNumber{2575}  % Insert your submission number here

\title{Unsupervised Selective Labeling for 
More Effective Semi-Supervised Learning}

\begin{comment}
\titlerunning{ECCV-22 submission ID \ECCVSubNumber} 
\authorrunning{ECCV-22 submission ID \ECCVSubNumber} 
\author{Anonymous ECCV submission}
\institute{Paper ID \ECCVSubNumber}
\end{comment}

\titlerunning{Unsupervised Selective Labeling}

\author{Xudong Wang\samethanks[1]\orcidID{0000-0002-4973-780X} \and
Long Lian\thanks{Equal contribution}\orcidID{0000-0001-6098-189X} \and
Stella X. Yu\orcidID{0000-0002-3507-5761}}
\authorrunning{X. Wang et al.}

\institute{UC Berkeley / ICSI
% \email{\{xdwang,longlian,stellayu\}@berkeley.edu}
}
\maketitle

\input{0abstract}
\input{1introduction}
\input{2method}
\input{3related_work}
\input{4experiments}
\input{5summary}

% \appendix
\input{6appendix}

\clearpage
\bibliographystyle{splncs04}
\bibliography{egbib}
\end{document}

%% file: math_commands.tex
\usepackage{amsmath,amsfonts,bm}

\def\eqref#1{equation~\ref{#1}}
\def\1{\bm{1}}

\newcommand{\test}{\mathcal{D_{\mathrm{test}}}}

\DeclareMathAlphabet{\mathsfit}{\encodingdefault}{\sfdefault}{m}{sl}
\SetMathAlphabet{\mathsfit}{bold}{\encodingdefault}{\sfdefault}{bx}{n}

\def\sA{{\mathbb{A}}}

\def\sD{{\mathbb{D}}}
\def\sS{{\mathbb{S}}}

\def\sV{{\mathbb{V}}}

\newcommand{\normltwo}{L^2}

\DeclareMathOperator*{\argmax}{arg\,max}
\DeclareMathOperator*{\argmin}{arg\,min}

%% file: 0abstract.tex
\begin{abstract}

Given an unlabeled dataset and an annotation budget, 
we study how to selectively label a fixed number of instances so that semi-supervised learning (SSL) on such a partially labeled dataset is most effective.  We focus on \textit{selecting} the right data to label, in addition to usual SSL's propagating labels from labeled data to the rest unlabeled data.
This instance selection task is challenging, as without any labeled data we do not know what the objective of learning should be.  Intuitively, no matter what the downstream task is, instances to be labeled must be {\it representative} and {\it diverse}:  The former would facilitate label propagation to unlabeled data, whereas the latter would ensure coverage of the entire dataset.
We capture this idea by selecting cluster prototypes, either in a pretrained feature space, or along with feature optimization, both without labels.
Our unsupervised selective labeling consistently improves SSL methods over state-of-the-art 
active learning given labeled data, by $8\app25\times$ in label efficiency. For example, it boosts FixMatch by 
10\% (14\%) in accuracy on CIFAR-10 (ImageNet-1K) with 0.08\% (0.2\%) labeled data,
 demonstrating that small computation spent on selecting what data to label brings significant gain especially under a low annotation budget.  Our work sets a new standard for practical and efficient SSL.

\keywords{semi-supervised learning \and unsupervised selective labeling}

\end{abstract}

%% file: 1introduction.tex
\section{Introduction}

%%%%
\def\smaller#1#{
    \fontsize{8.7}{11}
    #1
    \selectfont
}

\def\figSSLTask#1{
\begin{figure}[#1]
\centering
\includegraphics[width=0.98\linewidth]{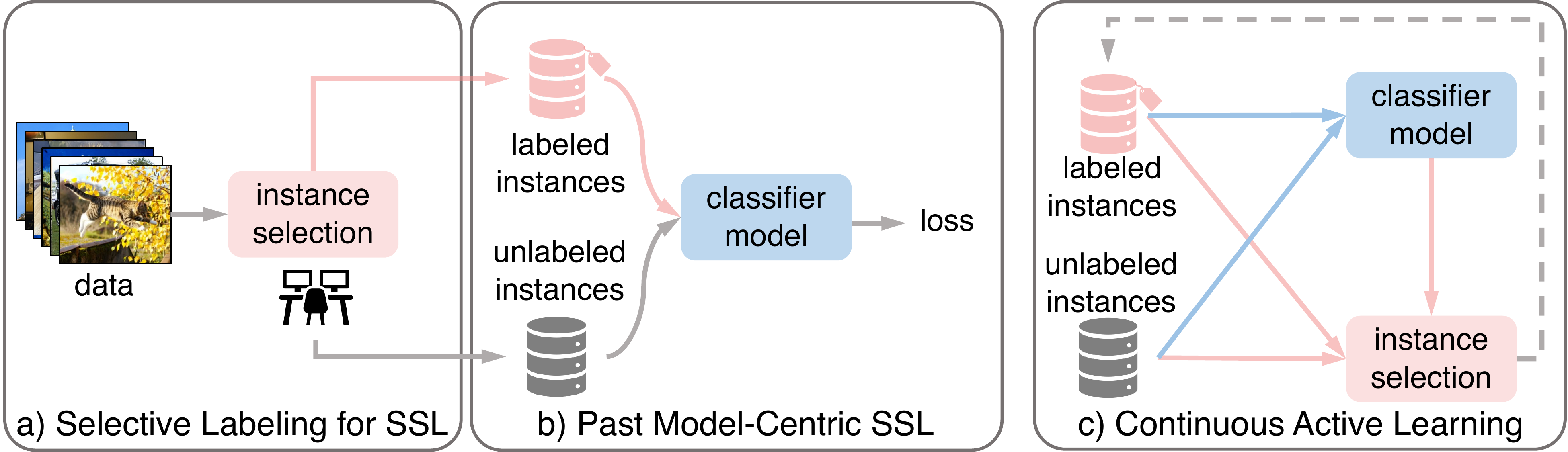}
\caption{
Our unsupervised selective labeling is a novel aspect of semi-supervised learning (SSL) and different from active learning (AL).
{\bf a, b)} Existing SSL methods focus on optimizing the model {\it given} labeled and unlabeled data. 
Instead of such model-centric learning, we focus on optimizing the selection of training instances {\it prior to} their label acquisition. 
{\bf c)} Existing AL methods alternate between classifier learning and instance selection, leveraging a classifier trained on initial labeled data and regularized on unlabeled data.
In contrast, we select instances from unlabeled data without knowing the classification task.
\label{fig:teaser-task}
}
\end{figure}
}

\iffalse
% With SSL
\def\figSSLTask#1{
\begin{figure}[#1]
\tablestyle{0pt}{0}
\begin{tabular}[t]{l}
\begin{subfigure}{1\linewidth}
    \centering
    \includegraphics[trim={10pt 0 10pt 0},clip,width=1.0\linewidth]{images/variation.pdf}
    \newsubcap{\smaller{Inter-fold vs. intra-fold (i.e. varying vs. fixed labeled samples) standard deviation of 5 runs on CIFAR-10 with 40 labeled samples and SimCLRv2-CLD \cite{chen2020big, wang2021unsupervised}.}}
    \label{fig:teaser-std}
\end{subfigure} \\
\begin{subfigure}{1.0\linewidth}
    \centering
    \includegraphics[width=1.0\linewidth]{images/teaserNew.pdf}
    \newsubcap{
    We consider the novel task of unsupervised data selection for semi-supervised learning (SSL).
    {\bf a, b)} Existing SSL methods focus on optimizing the model given labeled and unlabeled data, whereas we focus on optimizing labeled data selection for any given model. %
    {\bf c)} Existing AL methods, alternating between first classifier learning and then instance selection, focus on the supervised classification accuracy from labeled data (with regularization from unlabeled data), whereas we do not have any labeled data or any idea about the downstream classification task.
}
    \label{fig:teaser-task}
\end{subfigure}
\end{tabular}

\label{fig:teaser}
\end{figure}
}
\fi

\def\figTeaserTwo#1{
    \begin{figure*}[#1]\centering
\tb{@{}ccc@{}}{1}{
\tb{c}{0}{
\includegraphics[width=0.235\textwidth,height=0.25\textheight]{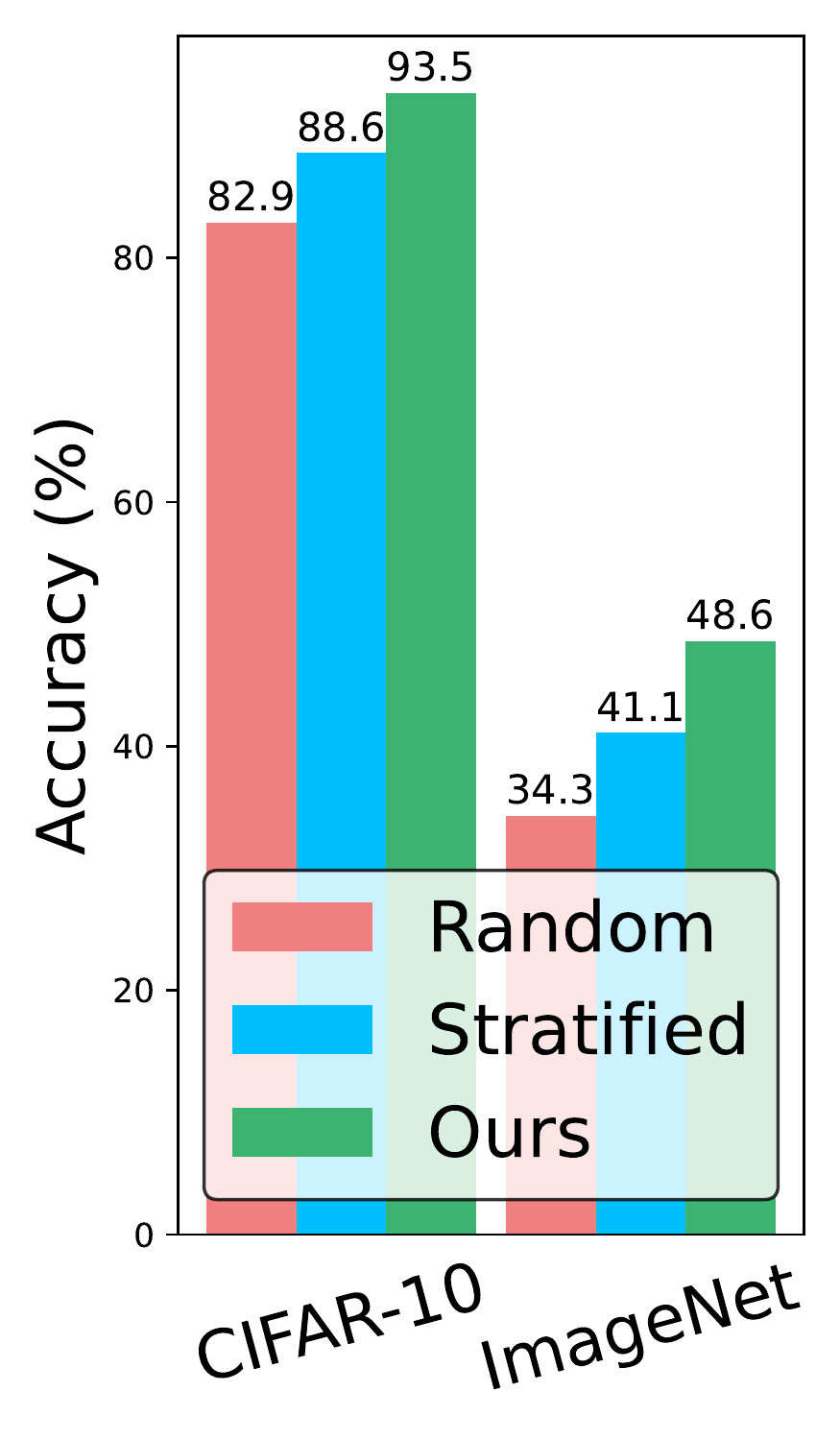}\\
{\bf a)} accuracy\\
} & 
\tb{c}{0}{
\includegraphics[width=0.24\textwidth,height=0.11\textheight]{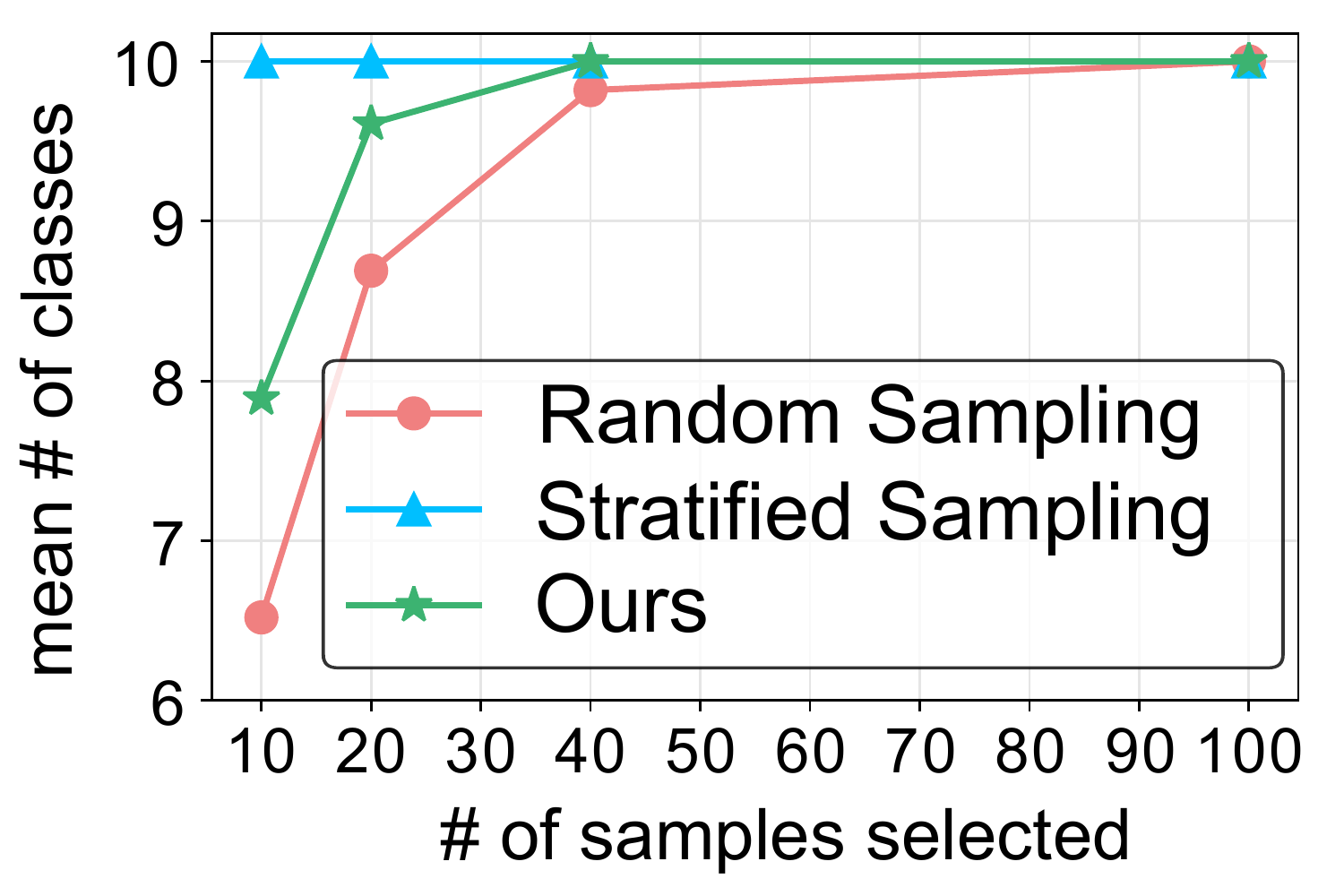}\\
{\bf b)} class coverage\\[2pt]
\includegraphics[width=0.24\textwidth,height=0.11\textheight]{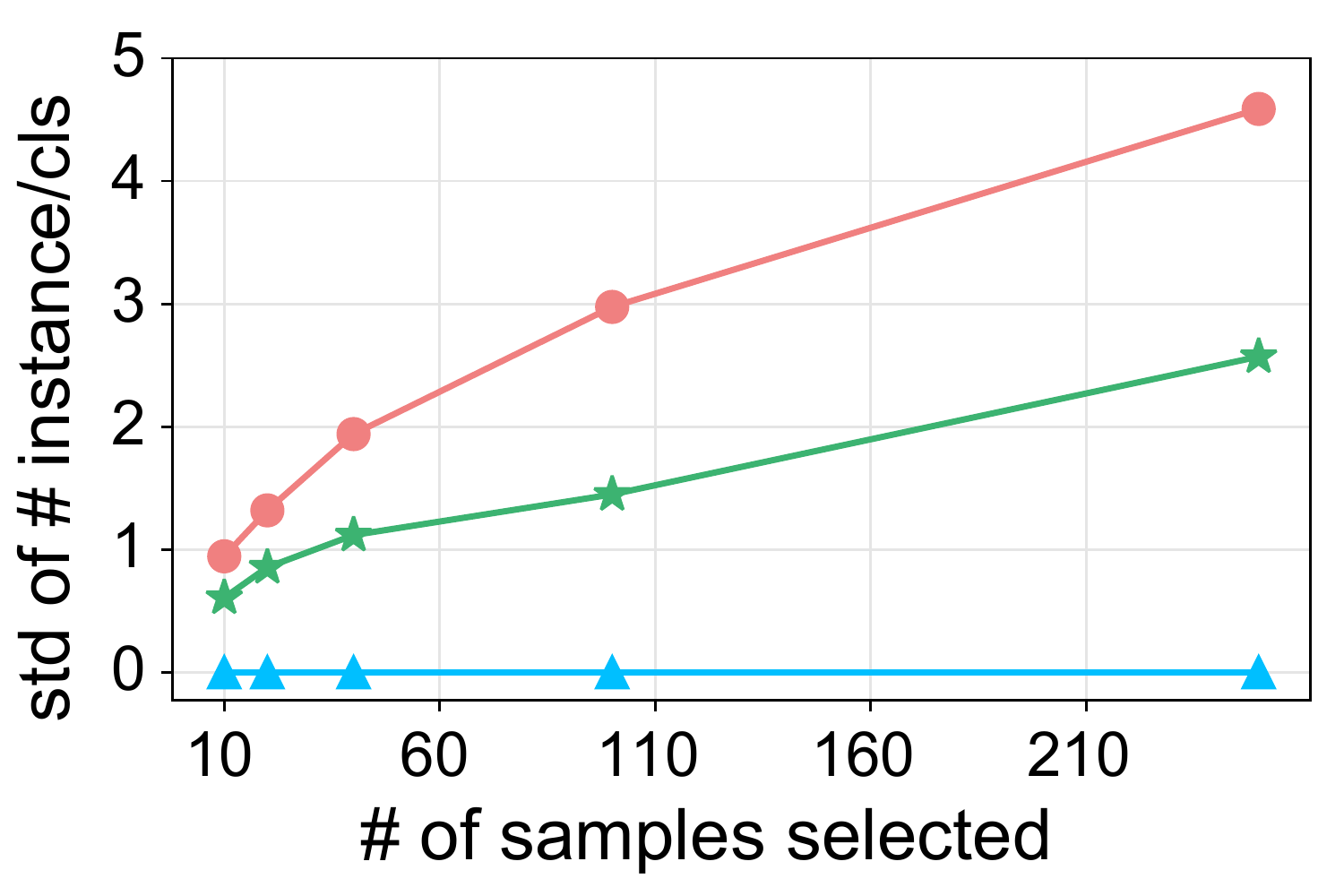}\\
{\bf c)} data balance
} & 
\tb{c}{0}{
\includegraphics[width=0.5\textwidth,height=0.23\textheight]{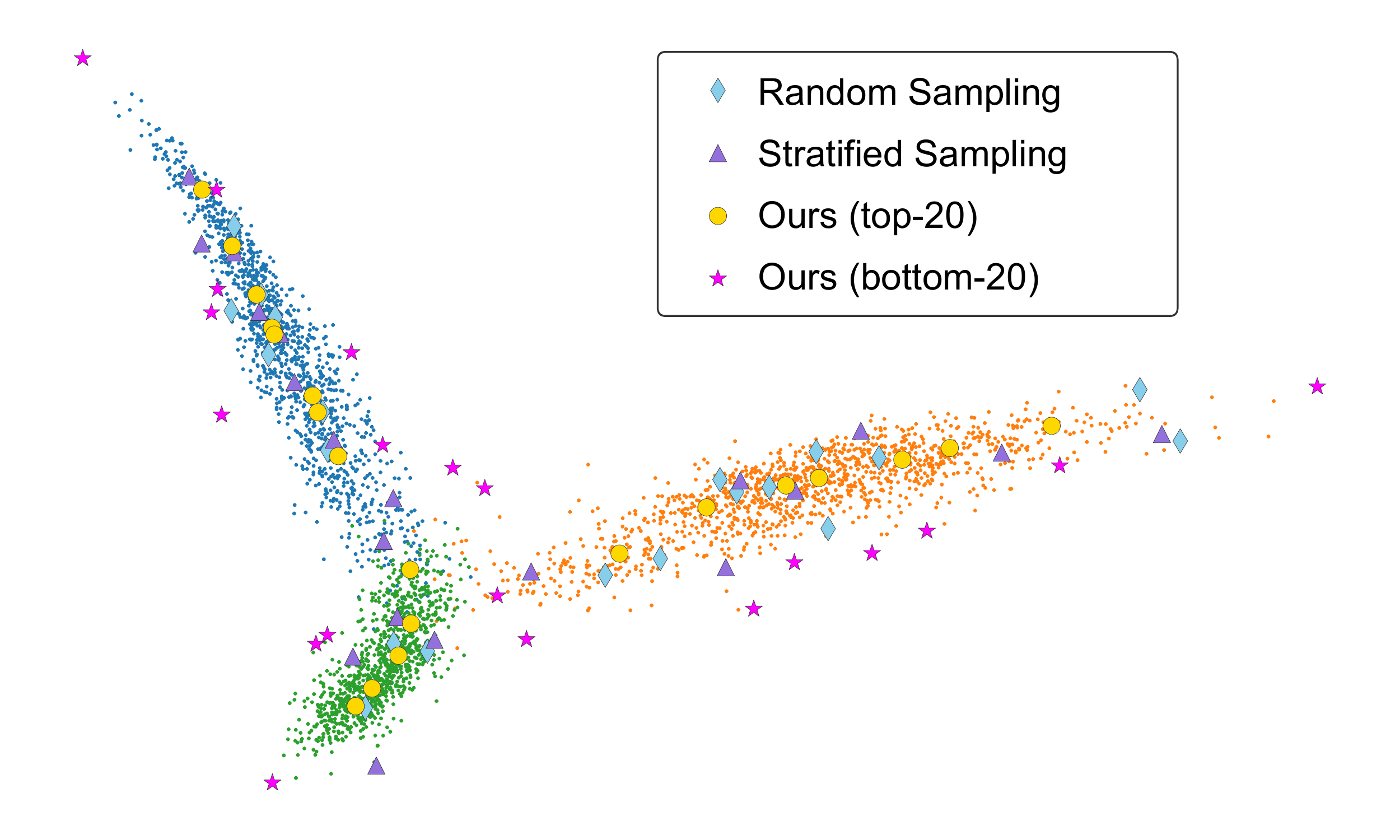}\\[9pt]
{\bf d)} informativeness
}\\
}
\caption{Our instance selection outperforms random and stratified sampling by  selecting a diverse set of representative instances. {\bf a)} The classification accuracy using SSL method FixMatch increases with our selectively labeled instances.  {\bf b)} Our method covers all the semantic classes with only a few instances.  {\bf c)} Our selection is far more balanced than random sampling.  {\bf d)} On a toy dataset of 3 classes in ImageNet, our top-ranked instances  cover informative samples across the entire space, whereas our bottom-ranked instances tend to be outliers. 
        }
        \label{fig:teasertwo}
    \end{figure*}
}

Deep learning's success on natural language understanding \cite{collobert2011natural}, visual object recognition \cite{krizhevsky2012imagenet}, and object detection \cite{girshick2014rich} follow a straightforward recipe: better model architectures, more data, and scalable computation \cite{lecun2015deep, goodfellow2016deep, hestness2017deep, xie2020self}.  As training datasets get bigger, their full task annotation becomes infeasible \cite{berthelot2019remixmatch, sohn2020fixmatch}.

Semi-supervised learning (SSL) deals with learning from both a small amount of labeled data {\it and} a large amount of unlabeled data: Labeled data directly supervise model learning, whereas unlabeled data help learn a desirable model that makes consistent \cite{sajjadi2016regularization, tarvainen2017mean, xie2020unsupervised, lee2013pseudo, berthelot2019mixmatch, berthelot2019remixmatch, sohn2020fixmatch, wang2022debiased} and unambiguous \cite{grandvalet2005semi, lee2013pseudo, berthelot2019mixmatch} predictions.  

Recent SSL methods approach fully supervised learning performance with a very small fraction of labeled data.  For example, on ImageNet, 
SSL with 1\% labeled data, i.e., only 13 instead of around 1300 labeled images per class, captures 95\% (76.6\% out of 80.5\% in terms of top-1 accuracy)  of supervised learning performance with 100\% fully labeled data \cite{chen2020big}.

The lower the annotation level, the more important what the labeled instances are to SSL.  While a typical image could represent many similar images, an odd-ball only represents itself, and labeled instances may even cover only part of the data variety, trapping a classifier in partial views with unstable learning and even model collapse.

% For example, if we fix the selection of labeled samples, the standard deviation of validation accuracy among 5 runs drop from 3.2\% to 0.4\% on CIFAR-10 with 40 labeled samples using SimCLRv2-CLD\cite{chen2020big,wang2021unsupervised}. 
%This indicates that to achieve a performant SSL pipeline with even higher sample efficiency, i.e. even fewer labeled samples, we need to pay more attention to the labeling process.

% However, for researchers and developers on SSL and many other areas, how to obtain labeled samples is often the less incentivized aspect, whereas building novel models and algorithms is often regarded as the lionized work. 

\figSSLTask{t!}
\figTeaserTwo{bt}

A common assumption in SSL is that labeled instances are sampled randomly either over all the available data or over individual classes, the latter known as stratified sampling \cite{berthelot2019mixmatch,berthelot2019remixmatch, sohn2020fixmatch, wang2022debiased}. Each method has its own caveats: Random sampling can fail to cover all semantic classes and lead to poor performance and instability,
% For example, random sampling often does not cover the whole 10 classes for the 40-label CIFAR-10 setting in \cite{sohn2020fixmatch}.
whereas stratified sampling is utterly unrealistic: If we can sample data by category, we would already have the label of every instance!

Selecting the right data to label for the sake of model optimization is not new. In fact, it is the focus of active learning (AL): Given an initial set of labeled data, the goal is to select an additional subset of data to label (\fig{teaser-task}) so that a model trained over such partially labeled data approaches that over the fully labeled data \cite{sener2017active,ducoffe2018adversarial,yoo2019learning}.  Unlabeled data can also be exploited for model training by combining AL and SSL,  resulting in a series of methods called semi-supervised active learning (SSAL).

However, existing AL/SSAL methods have several shortcomings.
\ol{enumerate}{0}{
\item They often require randomly sampled labeled data to begin with, which is sample-inefficient in low labeling settings that SSL methods excel at \cite{chan2021marginal}. 
\item AL/SSAL methods are designed with human annotators in a loop, working in multiple rounds of labeling and training. This could be cumbersome in low-shot scenario and leads to large labeling overhead.
\item AL's own training pipeline with a human-in-the-loop design makes its integration into existing SSL code implementation hard  \cite{song2019combining}.
\item The requested labels are tightly coupled with the model being trained so that labels need to be collected anew every time a model is trained with AL/SSAL.
}

\iffalse
However, most recent AL/SSAL methods require an initial random sampling stage, which shares a similar problem as the random sampler in SSL that they are often not data-efficient in very low label scenarios that SSL methods are already operating at, which is also illustrated at \cite{chan2021marginal}. 
% What is even worse is that the random selection takes a non-negligible part of the labeling budget in low-label scenario that SSL is operating at, leaving less flexibility in the later sample selection.
In addition, AL/SSAL methods are designed to work with human in the loop with multiple rounds of labeling and training, which could be cumbersome in low-label scenario with large labeling overhead.
% : For instance, instead of asking a person to label 40 samples, we now need to ask the person to label 14 random samples, 13 selected samples, and again 13 selected samples, in three \textit{sequential} queries. 
Furthermore, given that recent deep AL has its own training process, the human-in-the-loop design makes it difficult to integrate recent deep AL selection criteria into an existing SSL implementation \cite{song2019combining}.
\fi

We address {\it unsupervised selective labeling} for SSL (\fig{teaser-task}), in stark contrast with supervised data selection for AL, which is conditioned on an initial labeled set and for the benefit of a certain task.  Given only an annotation budget and an unlabeled dataset, among many possible ways to select a fixed number of instances for labeling, which way would lead to the best SSL model performance when it is trained on such partially labeled data? 

Our instance selection task is challenging, as without any labeled data we do not know what the objective of learning should be.  Intuitively, no matter what the downstream task is, instances to be labeled must be {\it representative} and {\it diverse}:  The former would facilitate label propagation to unlabeled data, whereas the latter would ensure coverage of the entire dataset.
We capture this idea by selecting cluster prototypes, either in a pretrained feature space, or along with feature optimization, both without labels.

\iffalse
% Probably want to clarify SSL pipeline
% Unsupervised Training-Free Selection (UTFS) and Unsupervised Training-Based Selection (UTBS)
To tackle this problem, two selective labeling methods for SSL are proposed:
Unsupervised Selective Labeling (USL) and training-based Unsupervised Selective Labeling (USL-T), both of which select samples to label in a purely unsupervised way. 
% Our methods fall into a category that are closely related to SSL and AL/SSAL. When compared to SSL, instead of being a competitor to recent SSL methods, our method chooses the labeled samples and can be \textit{added onto} SSL methods to improve their sample efficiency with a small change to the SSL pipeline. When compared to AL/SSAL, we could be considered as a variant of SSAL that does not waste the precious budget on random initial query and only require labeling once, but we are tailored to the extremely low-label setting of SSL that surpasses the regime of data efficiency of current AL/SSAL method. 
USL and USL-T are based on the intuition that ideal instances to be labeled shall collectively have maximum coverage, i.e. \textit{diversity}, for downstream classification tasks, and individually have maximum information propagation utility, i.e. \textit{representativeness}, for SSL. 
\fi

Our pipeline has three steps: 
1) Unsupervised feature learning that maps data into a discriminative feature space.
2) Select instances for labeling for maximum representativeness and diversity, without or with additional optimization.
3) Apply SSL (e.g., \cite{sohn2020fixmatch,chen2020big}) to the labeled data and the rest unlabeled data.

\fig{teasertwo} shows that our method has many benefits over random or stratified sampling for labeled data selection, in terms of accuracy, coverage, balance over classes, and  representativeness.  As it selects informative instances without initial labels, it can not only integrate readily into existing SSL methods, but also
achieve higher label efficiency than SSAL methods.
% The selection is done in a \textit{single} query, with makes it simple to integrate with existing SSL pipelines: the only change is the labeled subset of the training dataset. 
While most AL/SSAL methods only work on small-scale datasets such as CIFAR\cite{krizhevsky2009learning}, our method scales up easily to large-scale datasets such as ImageNet\cite{ILSVRC15}, taking less than an hour for our data selection on a commodity GPU server.

Our work sets a new standard for practical SSL with these contributions.
\olp{enumerate}{0}{
\item We systematically analyze the impact of different selective labeling methods on SSL under low-label settings, a previously ignored aspect of SSL. 
\item We propose two unsupervised selective labeling methods that capture representativeness and diversity without or along with feature optimization.
\item We benchmark extensively on our data selection with various SSL methods, delivering much higher sample efficiency over sampling in SSL or AL/SSAL.
\item We release our toolbox with AL/SSL implementations and a unified data loader, including benchmarks, selected instance indices, and pretrained models that combine  selective labeling with various methods for fair comparisons. 
}

%% file: 2method.tex
\def\figRegularization#1{
\begin{figure}[#1]\centering
\tb{@{}ccc@{}}{1}{
\includegraphics[width=0.33\linewidth]{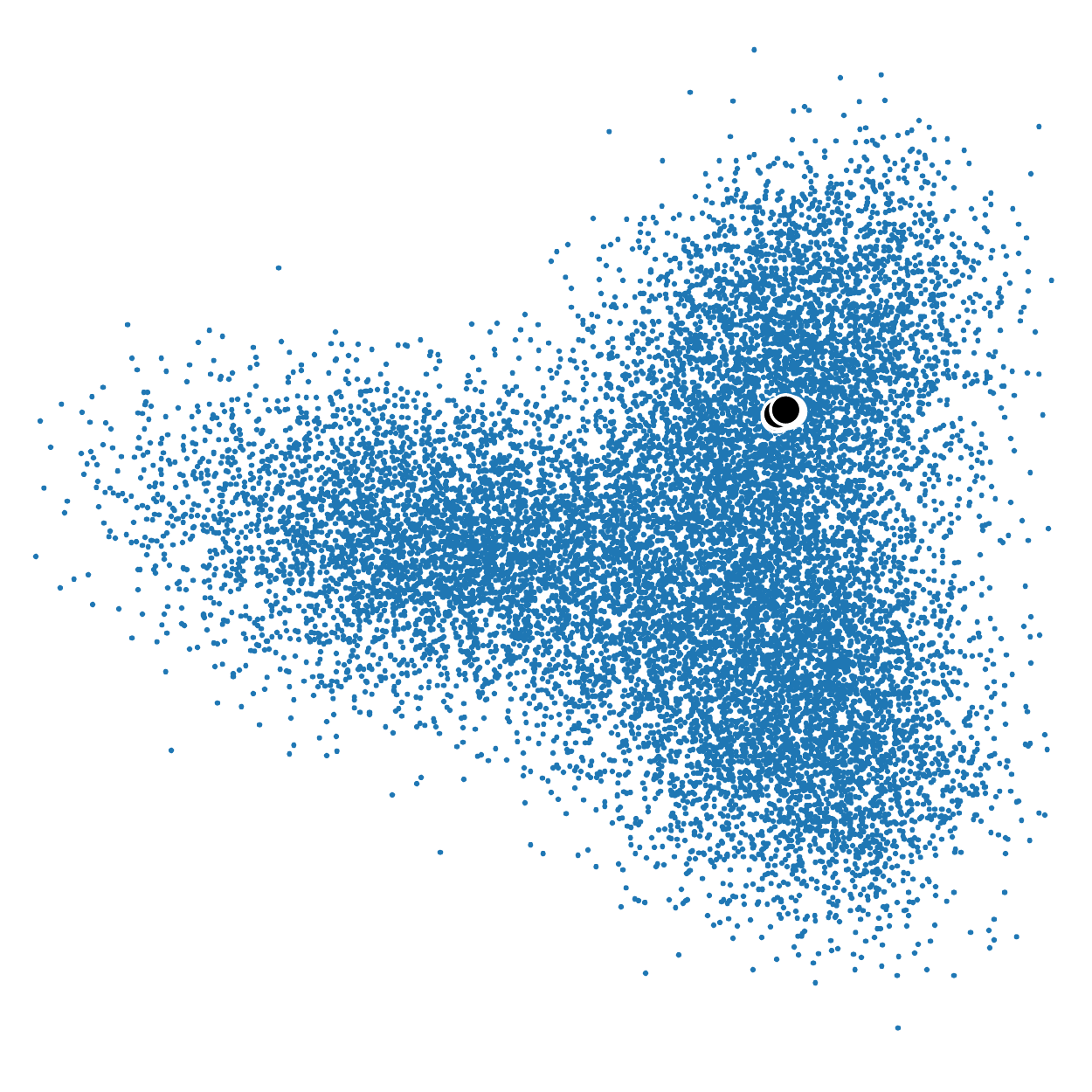} & 
\includegraphics[width=0.33\linewidth]{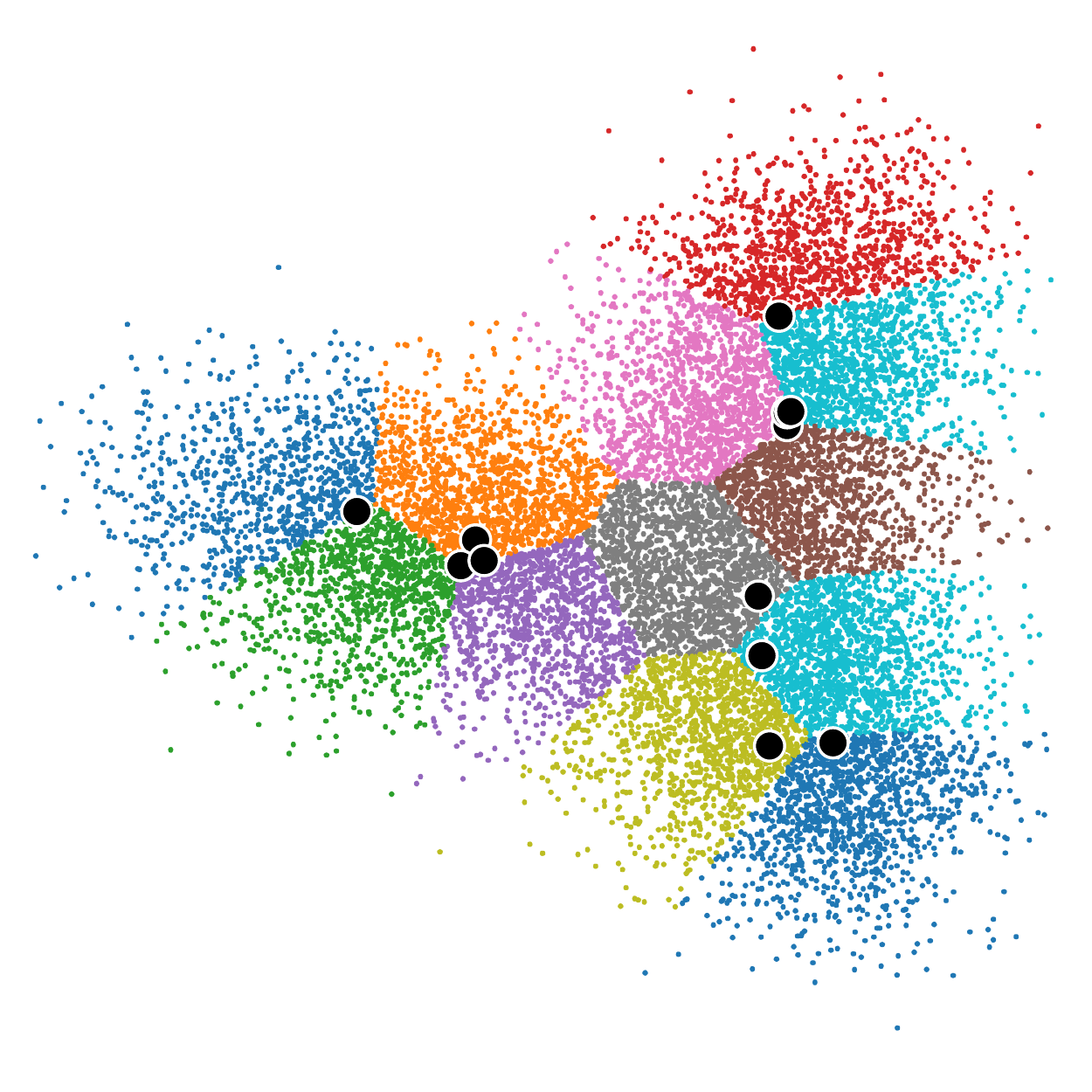}&
\includegraphics[width=0.33\linewidth]{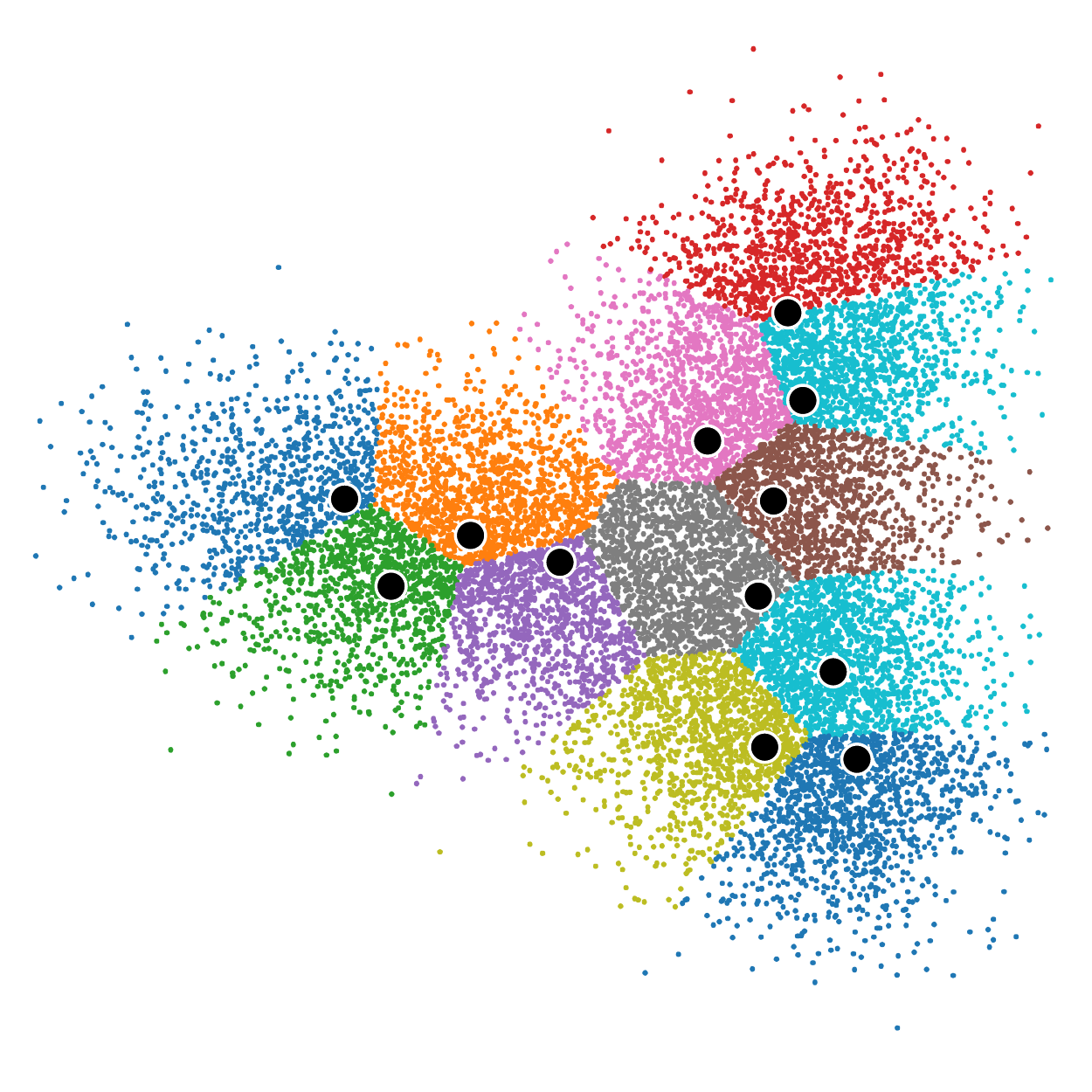}\\
{\bf a)} local only & 
{\bf b)} local $+$ global &
{\bf c)} local $+$ global $+$ reg.\\ 
}
\caption{
{\bf a)} Points at density peaks are individually representative of their local neighborhoods, but lack broad coverage of the entire set.
{\bf b)} Hard constraint by $K$-Means greatly depends on clustering quality and only partially alleviates the problem.
{\bf c)} Soft regularization leads to more uniform and diversified queries.
}
\label{fig:regularization}
\end{figure}
}

\def\tabComparison#1{
    \begin{table}[#1]\centering
    \tablestyle{3pt}{1.0}
    \normalsize
    \tablefontsmall{
    \begin{tabular}{l||c|c|c|c}
    \shline
    Property                        & ~\makecell{Semi-supervised \\Learning}~             & ~\makecell{Active \\Learning}~                & ~\makecell{Semi-supervised \\Active Learning}~          & ~Ours~ \\ \hline
    % Uses no initial random labels & {~\red{$\xmark$}}$^*$ & \red{$\xmark$}   & \red{$\xmark$}   & \green{$\cmark$} \\ \hline
    Uses no initial random labels & {\red{$\xmark$}} & \red{$\xmark$}   & \red{$\xmark$}   & \green{$\cmark$} \\ \hline
    Actively queries for labels            & \red{$\xmark$}    & \green{$\cmark$} & \green{$\cmark$} & \green{$\cmark$} \\ \hline
    Requires annotation only once          & \green{$\cmark$}  & \red{$\xmark$}   & \red{$\xmark$}   & \green{$\cmark$} \\ \hline
    Leverages unlabeled data               & \green{$\cmark$}  & \red{$\xmark$}   & \green{$\cmark$} & \green{$\cmark$} \\ \hline
    Allows label reuse across runs          & \green{$\cmark$} & \red{$\xmark$}   & \red{$\xmark$} & \green{$\cmark$} \\
    \shline
    \end{tabular}
    }
    \caption{Key properties of SSL, AL, SSAL, and our USL/USL-T pipelines. 
    Among them, our approach is the only one that does not use any random labels. %, which requires our model to understand the natural grouping of instances and extract semantic information without any labels. 
    % *: all labels are from random samples, except the ideal stratified sampling setting.
    }
    \label{tab:comparison}
    \end{table}
}

\section{Selective Labeling for Semi-supervised Learning}
\label{method}

Suppose we are given an unlabeled dataset of $n$ instances and an annotation budget of $m$.  Our task is to select $m$ ($m\!\ll\!n$) instances for labeling,  so that a SSL model trained on such a partially labeled dataset, with $m$ instances labeled and $n\!-\!m$ unlabeled, produces the best classification performance.

Formally, let $\sD = \{(x_i, y_i)\}_{i=1}^n$ denote $n$ pairs of image $x_i$ and its ({\it unknown}) class label $y_i$.  Let $\sA$ denote a size-$m$ subset of $\sD$ with {\it known} class labels.  Our goal is 
to select $\sA \!\subset\! \sD$ for acquiring class labels,  in order to maximize the performance of a given SSL model trained on labeled data $\sA$ and unlabeled data $\sD \!\setminus\! \sA$.

Our unsupervised selective labeling is challenging, as we do not have any labels to begin with, i.e., we don't know what would make the SSL model perform the best.  Our idea is to select $m$ instances that are not only \textit{representative} of most instances, but also \textit{diverse} enough to broadly cover the entire dataset, so that we do not lose information prematurely before label acquisition.

Our SSL pipeline with selective labeling consists of three steps: \textbf{1)} unsupervised feature learning; \textbf{2)} unsupervised instance selection for annotation; \textbf{3)} SSL on selected labeled data $\sA$ and remaining unlabeled data $\sD\!\setminus \sA$.

We propose two selective labeling methods in Step 2, training-free Unsupervised Selective Labeling (USL) and training-based Unsupervised Selective Labeling (USL-T), both aiming at selecting cluster prototypes in a discriminative feature space without label supervision.

\subsection{Unsupervised Representation Learning}
\label{sec:unsup_rep_learning}
Our first step is to obtain lower-dimensional and semantically meaningful features with unsupervised contrastive learning \cite{wu2018unsupervised,oord2018representation,he2020momentum,chen2020simple}, which maps $x_i$ onto a $d$-dimensional hypersphere with $\normltwo$ normalization, denoted as $f(x_i)$.
We use MoCov2 \cite{chen2020improved} (SimCLR \cite{chen2020simple} or CLD \cite{wang2021unsupervised}) to learn representations on ImageNet (CIFAR \cite{krizhevsky2009learning}).  See appendix for details.

\setlength{\belowdisplayskip}{3pt} \setlength{\belowdisplayshortskip}{3pt}
\setlength{\abovedisplayskip}{3pt} \setlength{\abovedisplayshortskip}{3pt}

\subsection{Unsupervised Selective Labeling (USL)}
We study the relationships between data instances using a weighted graph, where nodes $\{V_i\}$ denote data instances in the (normalized) feature space $\{f(x_i)\}$, and edges between nodes are attached with weights of pairwise feature similarity \cite{bondy1976graph,deo1975graph,chung1997spectral,shi2000normalized}, defined as $\frac{1}{D_{ij}}$, the inverse of feature distance $D$:
\begin{align}
D_{ij} = {\lVert f(x_i) - f(x_j) \rVert}.
\label{eqn:weight}
\end{align}
Intuitively, the smaller the feature distance, the better the class information can be transported from labeled nodes to unlabeled nodes.
Given a labeling budget of $m$ instances, 
we aim to select $m$ instances that are not only similar to others, but also well dispersed to cover the entire dataset.

\paragraph{{Representativeness}: Select Density Peaks.}
A straightforward approach is to select well connected nodes to spread semantic information to nearby nodes.  It corresponds to finding a density peak in the feature space. The  $K$-nearest neighbor density ($K$-NN) estimation \cite{fix1989discriminatory, orava2011k} is formulated as:
\begin{align}
{p}_{\text{KNN}} (V_i, k) = \frac{k}{n}\frac{1}{A_d\cdot D^d(V_i,V_{k(i)})}
\label{eqn:knn}
\end{align}
where $A_d={\pi^{d/2}}/{\Gamma(\frac{d}{2} + 1)}$ is the volume of a unit $d$-dimensional ball, $d$ the feature dimension, $\Gamma(x)$ the Gamma function, $k(i)$ instance $i$'s $k$th nearest neighbor.  
$p_{\text{KNN}}$ is very sensitive to noise, as it only takes the $k$th nearest neighbor into account. For robustness, we replace the $k$th neighbor distance $D(V_i,V_{k(i)})$ with the average distance $\bar{D}(V_i,k)$ to all $k$ nearest neighbors instead:
\begin{align}
\hat{p}_{\text{KNN}} (V_i, k) = \frac{k}{n}\frac{1}{A_d\cdot \bar{D}^d(V_i,k)}, \quad
\text{ where }
\bar{D}(V_i,k)=\frac{1}{k}\sum_{j=1}^k D(V_i,V_{j(i)}).
\label{eqn:utility}
\end{align}
We use $\hat{p}_{\text{KNN}} (V_i, k)$ to measure the {\it representativeness} of node $V_i$. Since only the relative ordering matters in our selection process, the density peak corresponds to the sample with maximum $\hat{p}_{\text{KNN}} (V_i, k)$ (i.e., maximum $1/\bar{D}(V_i,k)$).

\paragraph{{Diversity}: Pick One in Each Cluster.} While instances of high feature density values are individually representative, a separate criterion is necessary to avoid repeatedly picking similar instances near the same density peaks (Fig. \ref{fig:regularization}a).  
To select $m$ diverse instances that cover the entire unlabeled dataset, we resort to  $K$-Means clustering that partitions $n$ instances into $m(\leq\!n$) clusters, with each cluster represented by its centroid  $c$ \cite{lloyd1982least, forgy1965cluster} and every instance assigned to the cluster of the nearest centroid.  Formally, we seek $m$-way node partitioning $\sS=\{S_1, S_2, ..., S_m\}$ that minimizes the within-cluster sum of squares\cite{kriegel2017black}:
\begin{align}
    \begin{split}
    \!\!\min_{\sS}\sum_{i=1}^{m}\sum_{V\in S_i}\rVert V\!-\!c_i \rVert^2 = \min_{\sS}\sum_{i=1}^m|S_i|\text{Var}(S_i)
    \end{split}
    \label{eqn:kmeans}
\end{align}
It is optimized iteratively with EM \cite{mclachlan2007algorithm} from random initial centroids.  We then pick the most representative instance of each cluster according to Eqn.~\ref{eqn:utility}.

\figRegularization{!t}

\paragraph{Regularization: Inter-cluster Information Exchange.}
So far we use $K$-Means clustering to find $m$ hard clusters, and then choose the representative of each cluster {\it independently}.  This last step is sub-optimal, as instances of high density values could be located along cluster boundaries and close to instances in adjacent regions (Fig. \ref{fig:regularization}b).  We thus apply a regularizer to inform each cluster of other clusters' choices and iteratively diversify selected instances (Fig.~\ref{fig:regularization}c).

Specifically, let $\hat{\sV}^{t} = \{\hat{V}_1^{t},...,\hat{V}_m^{t}\}$ denote the set of $m$ instances selected at iteration $t$,
$\hat{V}_i^{t}$ for clusters $S_i$, where $i\!\in\!\{1,\ldots,m\}$.
For each candidate $V_i$ in cluster $S_i$, the farther it is away from those  in other clusters in $\hat{\sV}^{t-1}$, the more diversity it creates.  We thus minimize the total inverse distance to others in a regularization loss $\text{Reg}(V_i, t)$,  with a sensitivity hyperparameter $\alpha$:  
% which is set to 0.5 unless otherwise stated: 
\begin{align}
\text{Reg}(V_i, t) = \sum_{\hat{V}_j^{t-1}\not\in S_i} \frac{1}{\rVert V_i - \hat{V}_j^{t-1} \rVert^\alpha}.
\end{align}
This regularizer is updated with an exponential moving average:
\begin{align}
\overline{\text{Reg}}(V_i, t) = m_{\text{reg}} \cdot \overline{\text{Reg}}(V_i, t\!-\!1) \!+\!(1\!-\!m_{\text{reg}})\cdot \text{Reg}(V_i, t)
\label{regularization}
\end{align}
% The new samples are chosen by subtracting the regularizer for that sample multiplied by a hyperparameter $\lambda$ from the utility metric, where $\lambda$ controls the trade-off between diversity and informativeness: A low $\lambda$ leads to samples close to each other and of low diversity; a high $\lambda$ leads to uniformly distributed samples of low informativeness.
where $m_{\text{reg}}$ is the momentum. At iteration $t$, we select instance $i$ of the maximum {\it regularized utility} $U'(V_i, t)$ within each cluster:
\begin{align}
U'(V_i, t) = U(V_i) - \lambda\cdot\overline{\text{Reg}}(V_i, t)
\label{utility-reg}
\end{align}
where $\lambda$ is a hyperparameter that balances diversity and individual representativeness, utility $U(V_i)=1/\bar{D}(V_i,k)$.
% A low $\lambda$ leads to samples close to each other and of low diversity; a high $\lambda$ leads to uniformly distributed samples of low informativeness.
% The selection formed at the last iteration is our final choice. By adding a soft regularization, this formulation does not just push data point selections away to a certain distance, but it considers the trade-off between diversity and representativeness and attempts to find another mode in the cluster for potentially new information if labeled, thus greatly enhancing the overall informativeness. 
In practice, calculating distances between every candidate and every selected instance in $\hat{\sV}^{t-1}$ is no longer feasible for a large dataset,
% Since only selected samples that are close to the current candidate make a big difference in $\text{Reg}(V_i, t)$, for large datasets, 
so we only consider $h$ nearest neighbors in $\hat{\sV}^{t-1}$.  % In general, $h\in [32,128]$ performs well.
$\hat{\sV}^{t}$ at the last iteration is our final selection for labeling.

%%%%%%%%%%%%%%%%%%%%%%%%%%%%%%%%%%%%%%%%%%%%%%%%%%%%%%%%%%%%%%%%%%%%%%%%%%%%%%%%%%%%%%%%%%%%%%%%%%%%
\subsection{Training-Based Unsupervised Selective Labeling (USL-T)}
Our USL is a simple yet effective \textit{training-free} approach to selective labeling. Next we introduce an end-to-end \textit{training-based} Unsupervised Selective Labeling (USL-T), an alternative that integrates instance selection into representation learning and often leads to more balanced (Fig. \ref{fig:balance}) and more label-efficient (Table \ref{tab:cifar-10-al-ssal}) instance selection. 
% Furthermore, since the backbone is trainable in USL-T, hyperparams are very stable across datasets, reducing the effort needed in hyperparam tuning (Table \ref{tab:hyperparameters}).
The optimized model implicitly captures semantics and provides a strong initialization for downstream tasks (Sec. \ref{uslt_representation_learning}).

\paragraph{Global Constraint via Learnable $K$-Means Clustering.}
% The connections between this method with label-efficient learning.
% The purpose of the proposed deep $K$-Means clustering model is to provide a semantically meaningful feature space for unsupervised data-selection and a pseudo label for each sample. Therefore, we set up a global region constraint by giving each sample ${x}$ a cluster assignment. We acknowledge the natural grouping of instances by grouping similar instances together without supervision or prior knowledge of nature of the clusters. To achieve this goal, we designed a deep $K$-Means clustering algorithm with the theoretical guarantee that it shares the same global optimum with the traditional $K$-Means clustering, yet it aligns perfectly with the popular deep learning training practices in the form of a loss function, making it easy to implement.
Clustering in a given feature space is not trivial (Fig. \ref{fig:regularization}c).  We introduce a better alternative to $K$-Means clustering that jointly learns both the cluster assignment and the feature space for unsupervised instance selection.  

%Our deep $K$-Means clustering algorithm has the theoretical guarantee that it optimizes $K$-Means clustering's objective with a regularization term, yet it aligns perfectly with the popular deep learning methods in the form of a loss function, making it easy to implement.

Suppose that there are $C$ centroids initialized randomly.  For instance $x$ with feature $f(x)$, we infer one-hot cluster assignment distribution $y(x)$ by finding the closest {\it learnable} centroid ${c_i}, i\!\in\!\{1,\!\ldots\!, C\}$ based on feature similarity $s$:
\begin{align}
y_i(x) = 
\begin{cases}
    1,& \text{if } i = \arg\min_{k \in \{1, ..., C\}} s({f(x)}, {c_k})\\
    0,              & \text{otherwise}.
\end{cases}
\end{align}
\noindent We predict a soft cluster assignment $\hat{y}(x)$ by taking softmax over the similarity between instance $x$ and each learnable centroid:
\begin{align}
\hat{y}_i(x)= \frac{e^{s({f(x)}, {c_i})}}{\sum_{j=1}^C e^{s({f(x)}, {c_j})}}.
\end{align}
The hard assignment $y(x)$ can be regarded as pseudo-labels\cite{lee2013pseudo,sohn2020fixmatch,van2020scan}.
By minimizing $D_{\text{KL}}({y}(x) \| \hat {y}(x))$, the KL divergence between soft and hard assignments,  we encourage not only each instance to become more similar to its centroid, but also the learnable centroid to become a better representative of instances in the cluster.  With soft predictions, each instance has an effect on all the centroids.

Hardening soft assignments has a downside:  Initial mistakes are hard to correct with later training, degrading performance.  Our solution is to ignore ambiguous instances with maximal softmax scores below threshold $\tau$:
\begin{align}
% L_{\text{global}}(\{{x_i}\}_{i=1}^n) = \frac{1}{n}\sum_{i=1}^n D_{\text{KL}}({y}(x_i) || \hat {y}(x_i)) F(\hat {y}(x_i))
L_{\text{global}}(\{{x_i}\}_{i=1}^n) = \frac{1}{n}\sum_{\max( \hat {y}(x_i)) \geq \tau} D_{\text{KL}}({y}(x_i) \| \hat {y}(x_i))
\end{align}
where $\tau$ is the threshold hyper-parameter.  This loss leads to curriculum learning:  As instances are more confidently assigned to a cluster with more training, more instances get involved in shaping both feature $f(x)$ and clusters $\{c_i\}$.

Our global loss can be readily related to $K$-Means clustering.  
\begin{observation}
For $\tau\!=\!0$ and fixed feature $f$, optimizing $L_{\text{global}}$ is equivalent to optimizing $K$-Means clustering with a regularization term on inter-cluster distances that encourage additional diversity.
See Appendix for derivations.
\end{observation} 

\paragraph{Local Constraint with Neighbor Cluster Alignment.}
Our global constraint is the counterpart of $K$-Means clustering in USL. However, since soft assignments usually have low confidence scores for most instances at the beginning, convergence could be very slow and sometimes unattainable.
%Picking a low confidence threshold can alleviate this issue, however, at the cost of introducing noise in the assignment.  Therefore, 
We propose an additional local smoothness constraint by assigning an instance to the same cluster of its neighbors' in the unsupervisedly learned feature space to prepare confident predictions for the global constraint to take effect. 

This simple idea as is could lead to two types of collapses:  Predicting one big cluster for all the instances and predicting a soft assignment that is close to a uniform distribution for each instance. We tackle them separately.

\noindent
\textbf{1) For one-cluster collapse}, we adopt a trick for long-tailed recognition \cite{menon2020long} and
adjust logits to prevent their values from concentrating on one cluster: \begin{align}
\hat{P}({z}, \bar{z}) &= {z} -\alpha \cdot \log \bar z \\
\bar z & = \mu \cdot \sigma({z})+(1\!-\!\mu)\cdot \bar z
\end{align}
% $I({y}, \bar{y})\!=\!{y}\!-\alpha\!\log \bar y$,
where $\alpha$ controls the intensity of adjustment, $\bar{z}$ is an
exponential moving average of $\sigma({z})$, and $\sigma(\cdot)$ is the softmax function.

\noindent
\textbf{2) For even-distribution collapse}, we use a sharpening function \cite{berthelot2019mixmatch, berthelot2019remixmatch,assran2021semi} to encourage the cluster assignment to approach a one-hot probability distribution, where a temperature parameter $t$ determines the spikiness.

Both anti-collapse measures can be concisely captured in a single function $P(\cdot)$ that modifies and turns logits $z$ into a reference distribution:
\begin{align}
% P(z_i, \bar z_i, t)=\frac{\exp(\hat{P}(z_i, \bar{z_i})/t)}{\sum_j\exp(\hat{P}(z_j, \bar{z_j})/t)}
[P(z, \bar z, t)]_i &=\frac{\exp(\hat{P}(z_i, \bar{z}_i)/t)}{\sum_j\exp(\hat{P}(z_j, \bar{z}_j/t))}
\end{align}

We now impose our local labeling smoothness constraints with such modified soft assignments between $x_i$ and its randomly selected neighbor $x'_i$:
\begin{align}
% l_{\text{local}}({x}, x') = D_{\text{KL}}(I(y(x'), \bar y(x'), t) || \hat y(x)) \\
L_{\text{local}}(\{{x_i}\}_{i=1}^n) = \frac{1}{n} \sum_{i=1}^n D_{\text{KL}}(P(y(x'_i), \bar y(x'_i), t) || \hat y(x_i)).
\end{align}
We restrict $x'_i$ to $x$'s $k$ nearest neighbors,  selected according to the unsupervisedly learned feature prior to training and fixed for simplicity and efficiency.  

We show that our local constraint prevents both collapses.  
\begin{observation}
Neither \textit{one-cluster} nor \textit{even-distribution} collapse is optimal to our local constraint, i.e., $P(y(x'), \bar y(x'), t) \neq \hat y({x})$.  See Appendix for more details.
\end{observation}

Our final loss adds up the global and local terms with loss weight $\lambda$:
\begin{align}
L = L_{\text{global}} + \lambda L_{\text{local}}
\end{align}

\paragraph{Diverse and Representative Instance Selection in USL-T.}  Our USL-T is an end-to-end unsupervised feature learning method that directly outputs $m$ clusters for selecting $m$ {\it diverse} instances.   For each cluster, we then select the most {\it representative} instance, characterized by its highest confidence score, i.e. $\max \hat y({x})$.  Just as USL, USL-T improves model learning efficiency by  selecting diverse representative instances for labeling, without any label supervision.
% Unlike the soft constraint applied through regularization algorithm in USL, since $y(\cdot)$ in USL-T is a probability distribution based on optimizable features, we have already pushed samples that belong to other clusters away in the training process and thus samples that are selected in USL-T are typically very different from each other. 
% Therefore, the regularization term adopted in USL is no longer necessary for USL-T, which creates an end-to-end trained deep learning pipeline for selective labeling.
% Therefore, for simplicity, we do not include the regularization algorithm, which creates an end-to-end trained deep learning pipeline for selective labeling. 

\subsection{Distinctions and Connections With SSL/AL/SSAL}
\tabComparison{t!}
Table \ref{tab:comparison} compares our USL with related SSL, AL, and SSAL settings.
\ol{enumerate}{2}{
\item Our USL has the advantage of AL/SSAL that seeks optimal instances to label, yet does not require inefficient initial random samples or multiple rounds of human interventions.  USL has high label efficiency for selected instances in low label settings and does not need to trade off annotation budget allocation between initial random sampling and several interim annotation stages. 
\item 
Compared to AL, our USL also leverages unlabeled data. Compared to SSAL, USL is much easier to implement because we keep existing SSL implementation intact, while SSAL requires a human-in-the-loop pipeline.
Consequently,
unlike AL/SSAL where instance selection is coupled with the model to be trained, our selection is \textit{decoupled} from the downstream SSL model.  The same selection from USL works well even across different downstream SSL methods, enabling label reuse across different SSL experiments. 
\item
Most notably, our work is the first \textit{unsupervised} selective labeling method on large-scale recognition datasets that requests annotation only \textit{once}.
}

%% file: 3related_work.tex
\section{Related Work}
\label{related_work}

% Add citation as reviewer requested

\paragraph{Semi-supervised Learning}(SSL) integrates information from small-scale labeled data and large-scale unlabeled data.
\textit{Consistency-based regularization} \cite{sajjadi2016regularization, tarvainen2017mean, xie2020unsupervised} applies a consistency loss by imposing invariance on unlabeled data under augmentations. % by leveraging the idea that after applying data augmentations, perturbations or adding adversarial noise to unlabeled data, the semantic distribution of the classifier output remains unchanged.
\textit{Pseudo-labeling} \cite{lee2013pseudo, berthelot2019mixmatch, berthelot2019remixmatch, wang2022debiased} relies on the model’s high confidence predictions to produce pseudo-labels of unlabeled data and trains them jointly with labeled data. % which can be added to the training data set as labeled data. 
FixMatch \cite{sohn2020fixmatch} integrates strong data augmentation \cite{cubuk2020randaugment} and pseudo-label filtering\cite{liu2019deep} and explores training on the most representative samples ranked by \cite{carlini2019distribution}. However, \cite{carlini2019distribution} is a supervised method that requires all labels. % and a model optimized on them.
\textit{Transfer learning} method SimCLRv2 \cite{chen2020big} is a two-stage SSL method that applies contrastive learning followed by fine-tuning on labeled data. 
\textit{Entropy-minimization} \cite{grandvalet2005semi, berthelot2019mixmatch} assumes that classification boundaries do not pass through the high-density area of marginal distributions and enforces confident predictions on unlabeled data. 
Instead of competing with existing SSL methods, our USL enables more effective SSL by choosing the right instances to label {\it for} SSL, without any prior semantic supervision.

\paragraph{Active Learning}(AL) aims to select a small subset of labeled data  to achieve competitive performance over supervised learning on fully labeled data \cite{cohn1994improving, roy2001toward, bilgic2009link}. \textit{Traditional AL} has three major types \cite{settles2009active, ren2020survey}: 
membership query synthesis \cite{angluin1988queries}, stream-based selective sampling \cite{dagan1995committee, atlas1990training}, and pool-based active learning \cite{tong2001support, huang2010active, wei2015submodularity, miao2021iterative}.
In \textit{Deep AL},
Core-Set \cite{sener2017active} approaches data selection as a set cover problem. % and is equivalent to the k-Center problem with a derived upper bound.
\cite{ducoffe2018adversarial} estimates distances from decision boundaries based on sensitivity to adversarial attacks.
LLAL \cite{yoo2019learning} predicts target loss of unlabeled data parametrically
and queries instances with the largest loss for labels.
\textit{Semi-supervised Active Learning} (SSAL) combines AL with SSL. \cite{song2019combining} merges uncertainty-based metrics with MixMatch\cite{berthelot2019mixmatch}. \cite{gao2020consistency} merges consistency-based metrics with consistency-based SSL. AL/SSAL often rely on initial labeled data to learn both the model and the instance sampler, requiring multiple (e.g. 10) rounds of sequential annotation and significant modifications of existing annotation pipelines. 
% We compare extensively with AL/SSAL in Sec.~\ref{experiment}. 
Recent \textit{few-label transfer}\cite{li2021improve} leverages features from a large source dataset to select instances in a smaller target dataset for annotation. It also requires a seed instance per class to be pre-labeled in the target dataset, whereas we do not need supervision anywhere for our instance selection.

\paragraph{Deep Clustering.} DeepCluster \cite{caron2018deep} also jointly learns features and cluster assignments with $k$-Means clustering. However, USL-T, with end-to-end backprop to jointly optimize classifiers and cluster assignments, is much more \textit{scalable} and \textit{easy to implement}. UIC/DINO \cite{chen2020unsupervised,caron2021emerging} incorporate neural networks with categorical outputs through softmax, but both methods focus on learning feature or attention maps for downstream applications instead of acquiring a set of instances that are representative and diverse. Recently, SCAN/NNM/RUC \cite{van2020scan,dang2021nearest,park2021improving} produce image clusters to be evaluated against semantic classes via Hungarian matching. However, such methods are often compared {\it against} SSL methods \cite{van2020scan}, whereas our work is  {\it for} SSL methods.  See appendix for more discussions about \textbf{self-supervised learning} and \textbf{deep clustering} methods.

%% file: 4experiments.tex
\section{Experiments}
\label{experiment}

% Table 2: Main CIFAR-10 (compare with AL/SSAL)
\def\tabCifarALSSAL#1{
    \centering
    \tablestyle{2pt}{0.9}
    \normalsize
    \tablefontsmall{
        \begin{tabular}{l|r|r}
        \shline
        \textbf{CIFAR-10} & Budget & Acc (\%) \\
        \shline
        \rowcolor{LightGray} \multicolumn{3}{l}{\textit{Active Learning (AL)}} \\
        %Entropy \\
        CoreSet\cite{sener2017active}$^\dagger$ & 7500 & 85.4 \\ % TOD
        %Uniform \\
        VAAL\cite{sinha2019variational}$^\dagger$ & 7500 & 86.8 \\ % TOD
        % MCDropout\cite{gal2017deep}$^\dagger$ & 7500 & 85.1 \\ % TOD
        % TA-VAAL\cite{kim2021task}$^\dagger$ & 7500 & 84.3 \\ % TOD
        % LLAL\cite{yoo2019learning}$^\dagger$ & 7500 & 86.6 \\ % TOD
        % SRAAL\cite{zhang2020state}$^\dagger$ & 7500 & 86.9 \\ % TOD
        UncertainGCN\cite{caramalau2021sequential}$^\dagger$ & 7500 & 86.8 \\ % TOD
        CoreGCN\cite{caramalau2021sequential}$^\dagger$ & 7500 & 86.5 \\ % TOD
        % TOD AL\cite{huang2021semi} & 7500 & 86.9 \\ % TOD
        MCDAL\cite{cho2021mcdal} & 7500 & 87.2 \\ % MCDAL
        \shline
        \rowcolor{LightGray} \multicolumn{3}{l}{\textit{Semi-supervised Active Learning (SSAL)}} \\
        % Entropy$\ddagger$ & 250 & 87.8 \\ % CBSSAL, but not proposed in CBSSAL
        TOD-Semi\cite{huang2021semi} & 7500 & 87.8 \\ % TOD
        CoreSetSSL\cite{sener2017active}$^\ddagger$ & 250 & 88.8 \\ % CBSSAL
        CBSSAL \cite{gao2020consistency} & 150 & 87.6 \\
        % CBSSAL & 250 & 90.2 \\ % CBSSAL
        MMA\cite{song2019combining} & 500 & 91.7 \\ % MMA
        MMA+k-means\cite{song2019combining} & 500 & 91.5 \\ % MMA
        REVIVAL\cite{guo2021semi} & 150 & 88.0 \\ % REVIVAL
        \shline
        \rowcolor{LightGray} \multicolumn{3}{l}{\textit{Selective Labeling}} \\
        \rowcolor{LightGreen}
        FixMatch + USL (Ours) & \bf 40 & 90.4 \\
        \rowcolor{LightGreen}
        FixMatch + USL (Ours) & \bf 100 & \bf 93.2 \\
        \rowcolor{LightGreen}
        FixMatch + USL-T (Ours) & \bf 40 & \bf 93.5 \\
        \shline
        \end{tabular}
    }
    % \captionof{table}{Our selective labeling methods, USL and USL-T, greatly outperform AL/SSAL in accuracy and label efficiency on CIFAR-10. $\dagger$, $\ddagger$: results from \cite{huang2021semi} and \cite{gao2020consistency}, respectively.}
    \captionof{table}{USL and USL-T greatly outperform AL/SSAL methods in accuracy and label efficiency on CIFAR-10. $\dagger$, $\ddagger$: results from \cite{huang2021semi} and \cite{gao2020consistency}, respectively.}
    \label{tab:cifar-10-al-ssal}
}

% Figure 4: compare with AL/SSAL (on the right of Table 2)
\def\figALSSAL#1{
\setlength{\tabcolsep}{0.5mm}
\begin{tabular}{c}
    \begin{subfigure}{1.0\textwidth}
        \centering
        \includegraphics[width=0.85\linewidth]{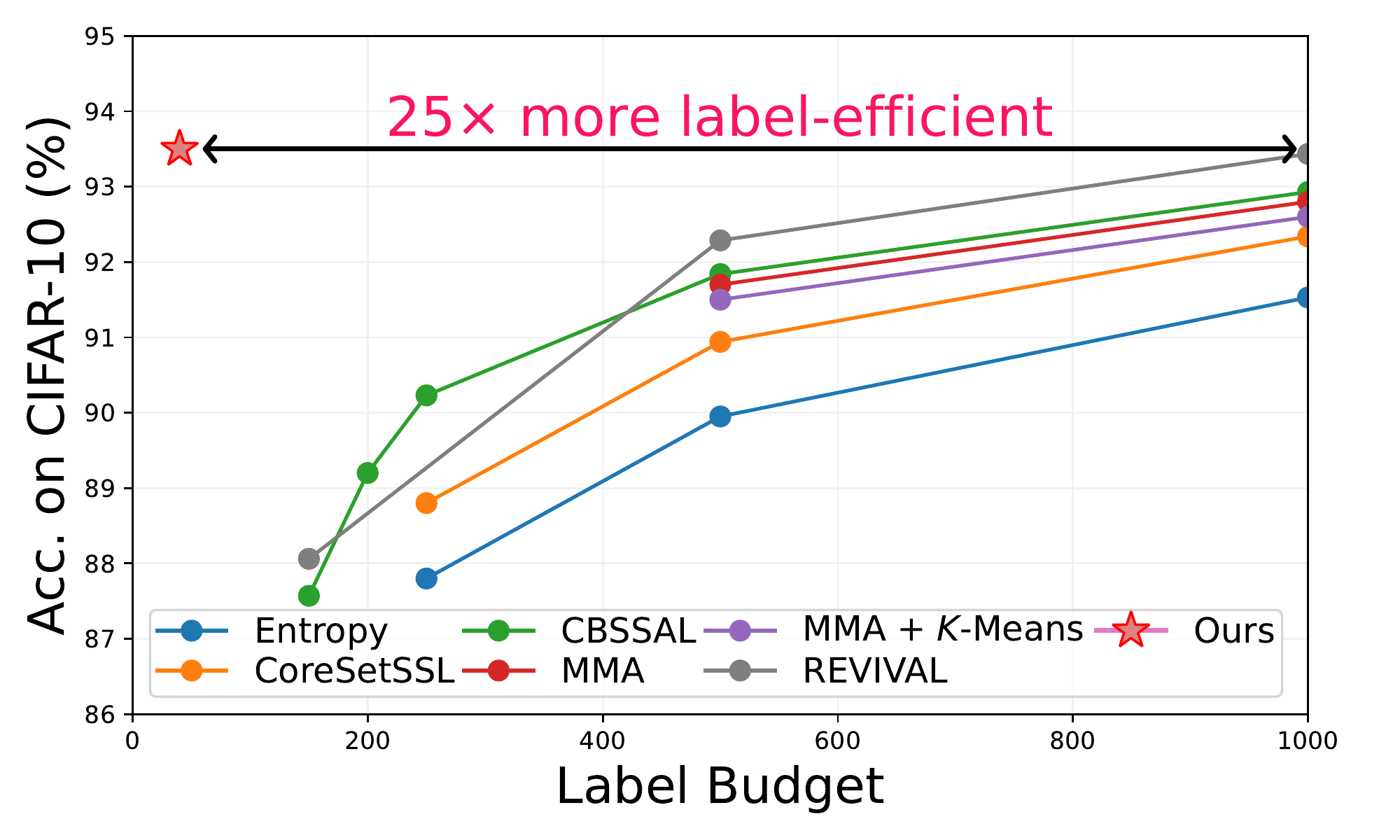}
        \caption{CIFAR-10}
        \label{fig:cifar10_ssal_comp}
    \end{subfigure} 
    \\
    \begin{subfigure}{1.0\textwidth}
        \centering
        \includegraphics[width=0.85\linewidth]{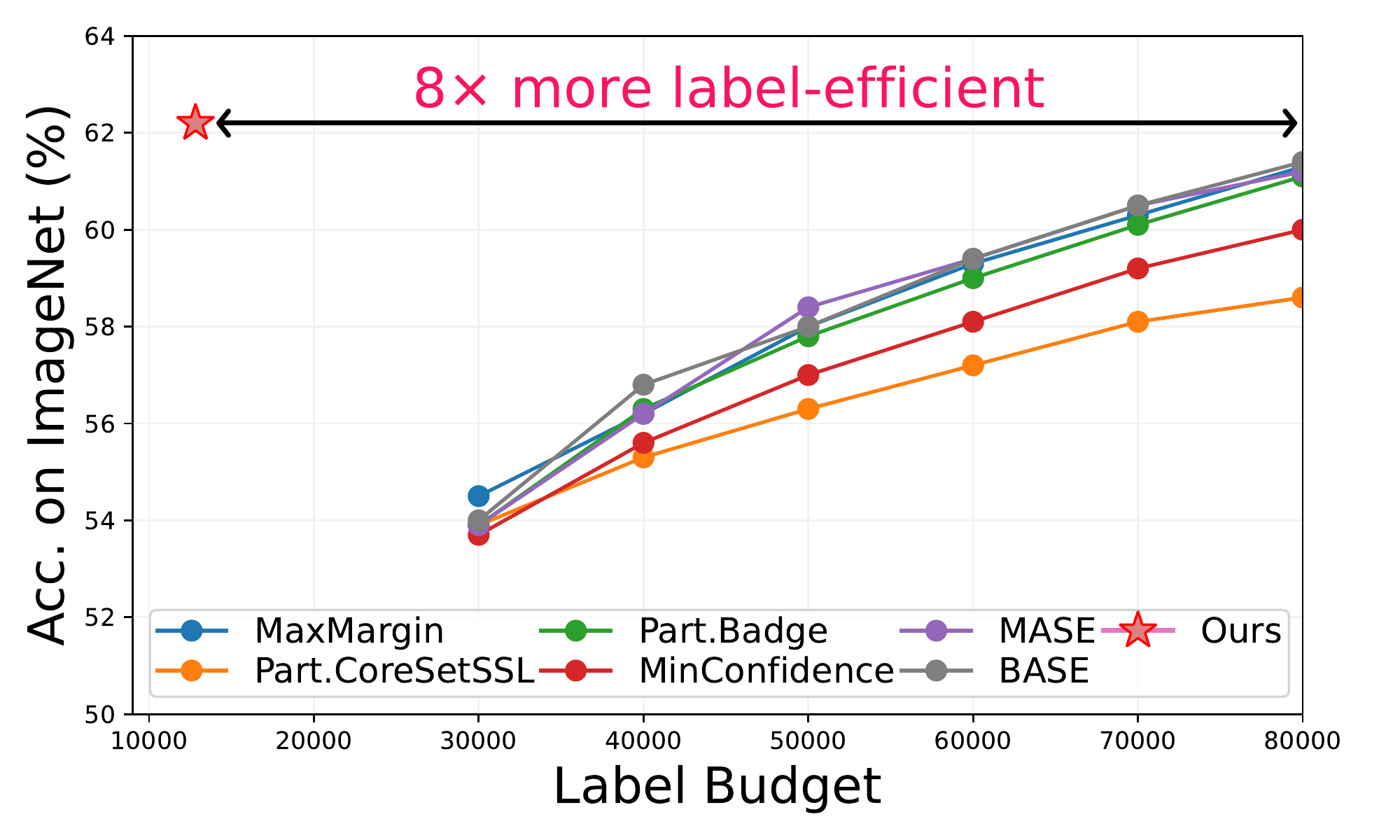}
        \caption{ImageNet}
        \label{fig:imagenet_ssal_comp}
    \end{subfigure}
\end{tabular}
\caption{Compared to SSAL, USL gets up to $25\x$ higher label efficiency. % on CIFAR-10 and ImageNet.
}
\label{fig:al_ssal}
}

% The full command for table and figure
\def\tabfigALSSAL#1{
\begin{figure}[#1]
\begin{minipage}[t!]{0.5\textwidth}
    \centering
    \tabCifarALSSAL{c}
\end{minipage}
\hfill
\begin{minipage}[t!]{0.46\textwidth}
    \centering
    \figALSSAL{c}
\end{minipage}
\end{figure}
}

% Table 3: Main CIFAR-10 on various base methods
\def\tabCifarSSL#1{
    \centering
    \tablestyle{4pt}{0.9}
    \normalsize
    \tablefontsmall{
    \begin{tabular}{l|l|l}
    \shline
    \textbf{CIFAR-10} & S.v2-CLD & FixMatch \\
    \shline
    Random Selection & 60.8 & 82.9 \\
    \textcolor{gray}{Stratified Selection$^\dagger$} & \textcolor{gray}{66.5} & \textcolor{gray}{88.6} \\
    \hline
    UncertainGCN & 63.0 & 77.3 \\
    CoreGCN & 62.9 & 72.9 \\
    MMA$^{+\ddagger}$ & 60.2 & 71.3 \\
    TOD-Semi & 65.1 & 83.3 \\
    \hline
    \rowcolor{LightGreen}
    USL (Ours) &\bf  76.6 \textbf{\color{Green}{$\uparrow$11.5}} &\bf  90.4 \textbf{\color{Green}{$\uparrow$7.1}} \\
    \rowcolor{LightGreen}
    USL-T (Ours) &\bf  76.1 \textbf{\color{Green}{$\uparrow$11.0}} &\bf  93.5 \textbf{\color{Green}{$\uparrow$10.2}} \\
    \shline
    \end{tabular}
    }
    \captionof{table}{The samples selected by USL and USL-T greatly outperform the ones from AL/SSAL on \cite{sohn2020fixmatch, chen2020big, wang2021unsupervised}, with a budget of 40 labels on CIFAR-10. 
    % See appendix for other settings.
    $^\ddagger$: MMA$^+$ is our improved MMA\cite{song2019combining} based on FixMatch.
    % but its selection is highly imbalanced and seldom covers all classes (see Fig. \ref{fig:cifar10-balance}), causing worse-than-random performance.
    $\dagger$: not a fair baseline.
    }
    \label{tab:cifar-10-ssl}
}

% Figure 5: Balance
\def\figCifarALSSALBalance#1{
    \setlength{\tabcolsep}{0.5mm}
    \begin{subfigure}{1.0\textwidth}
        \centering
        \includegraphics[width=1.0\linewidth, height=0.155\textheight]{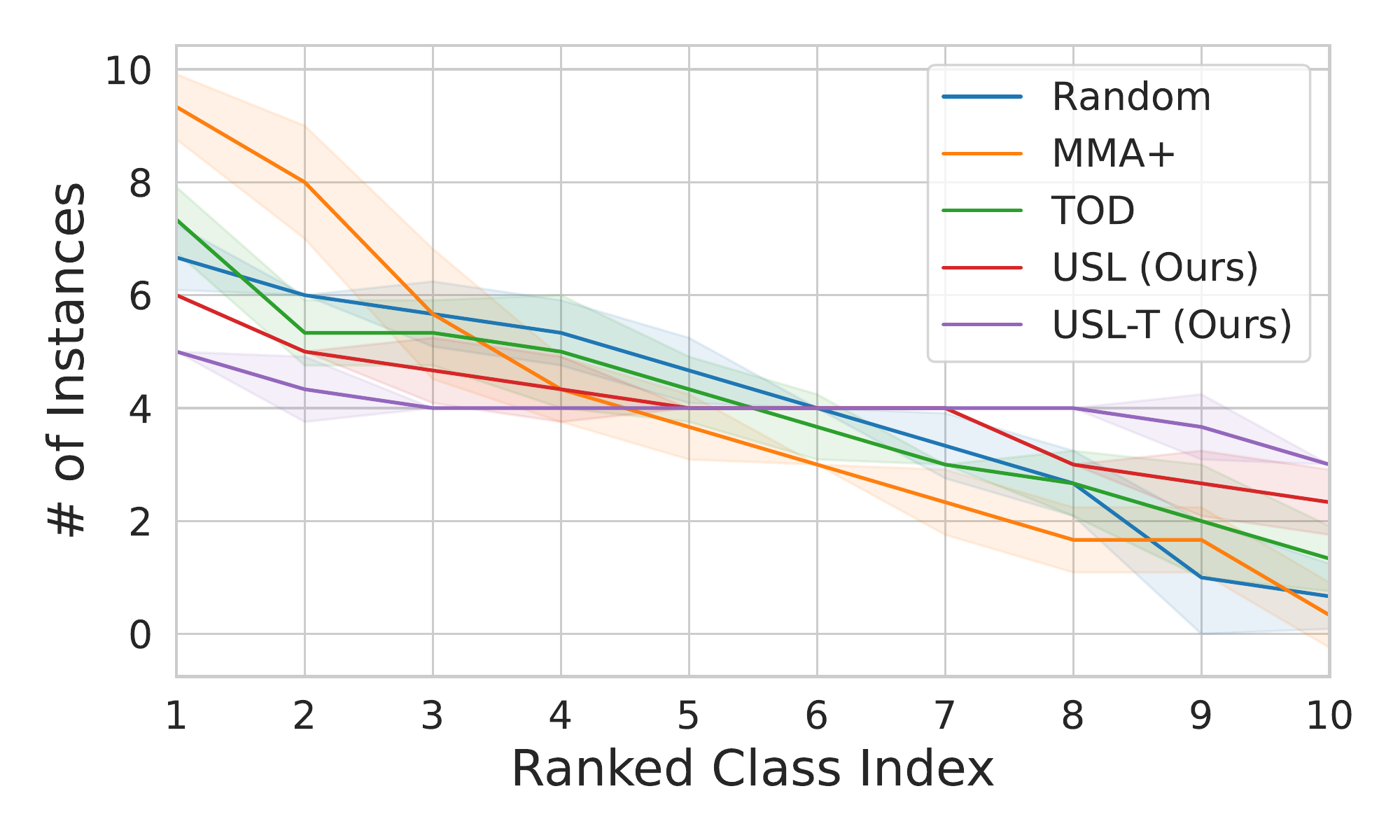}
    \end{subfigure}
    \caption{Comparisons on the semantic class distributions of several methods over 3 runs. 
    USL and USL-T get more balanced distribution.
    % , whereas samples from MMA$^+$ and TOD are often even more imbalanced than randomly selected ones.
    \label{fig:balance}
    }
}

\def\figCifarALSSALBalanceOriginal#1{
    \setlength{\tabcolsep}{0.5mm}
    \begin{tabular}{c}
        \begin{subfigure}{1.0\textwidth}
            \centering
            %\smallskip
            \includegraphics[width=0.86\linewidth]{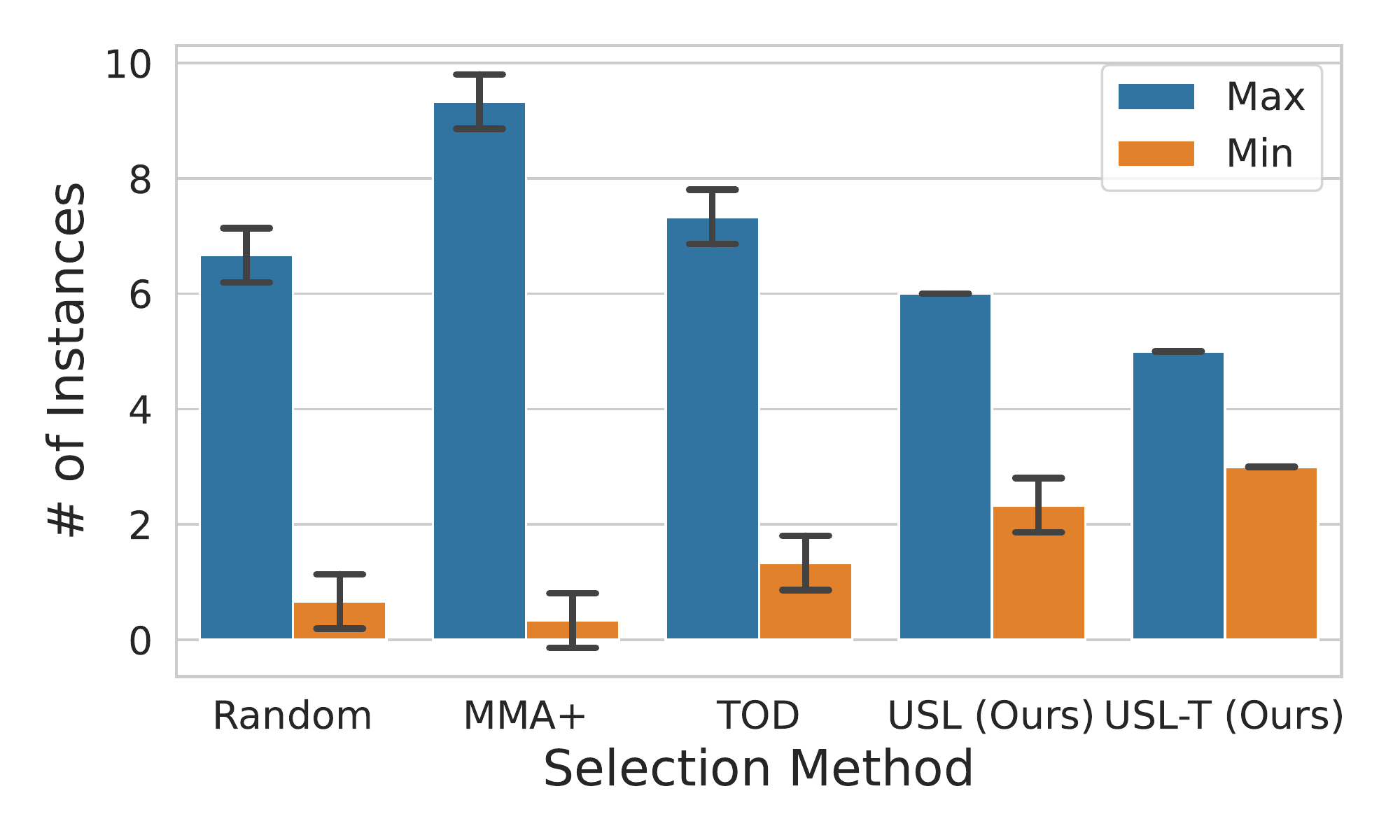}
            \captionof{figure}{Min/Max of Distribution}
            \label{fig:cifar10-balance-max_min}
        \end{subfigure} \\
        \begin{subfigure}{1.0\textwidth}
            \centering
            %\smallskip
            \includegraphics[width=0.86\linewidth]{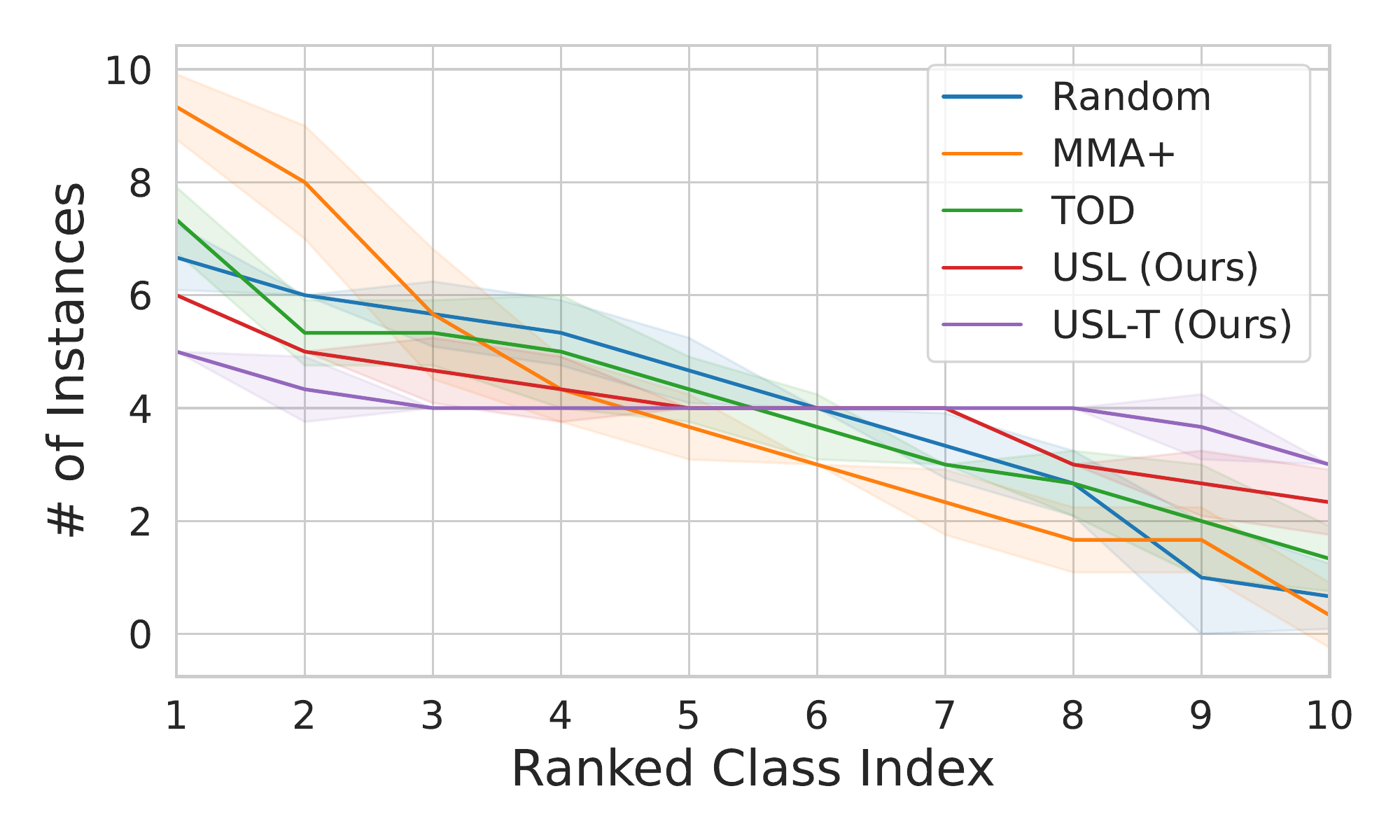}
            \captionof{figure}{Full Sorted Distribution}
            \label{fig:cifar10-balance-dist}
        \end{subfigure}
    \end{tabular}
    \caption{Comparisons on the semantic class distributions of several methods over 3 runs. 
    USL and USL-T selections have more balanced semantic distribution, whereas samples from MMA$^+$ and TOD are often even more imbalanced than randomly selected ones.
    \label{fig:cifar10-balance}
    }
}

% The full command for table and figure
\def\tabfigCifarSSLBalance#1{
\begin{figure}[#1]
\begin{minipage}[t]{0.59\textwidth}
    \centering
    \tabCifarSSL{t}
\end{minipage}
\hfill
\begin{minipage}[t]{0.38\textwidth}
    \centering
    \figCifarALSSALBalance{t}
\end{minipage}
\end{figure}
}

% CIFAR 10: Additional
\def\tabCifarAllMethods#1{
    \begin{table}[#1]
    \centering
    \tablestyle{4pt}{0.9}
    \normalsize
    \tablefontsmall{
        \begin{tabular}{l|l|l|l|l|l}
        \shline
        % & \multicolumn{4}{c}{Accuracy (\%)} \\
        \textbf{CIFAR-10} & \multicolumn{1}{c|}{MixMatch} & \multicolumn{1}{c|}{SimCLRv2} & \multicolumn{1}{c|}{SimCLRv2-CLD} & \multicolumn{1}{c|}{FixMatch} & \multicolumn{1}{c}{CoMatch} \\
        \shline
        Random & 43.4 & 55.9 & 60.8 & 82.9 & 87.4 \\
        \textcolor{gray}{Stratified$\dagger$ } & 
            \textcolor{gray}{62.0} & 
            \textcolor{gray}{69.8} & 
            \textcolor{gray}{66.5} & 
            \textcolor{gray}{88.6} & 
            \textcolor{gray}{93.1}
            \\
        \hline
        \rowcolor{LightGreen}
        USL (Ours) & 
        \textbf{61.6} \textbf{\color{Green}{$\uparrow$18.2}} & 
        \textbf{69.1} \textbf{\color{Green}{$\uparrow$13.2}} & 
        \textbf{76.6} \textbf{\color{Green}{$\uparrow$15.8}} &
        \textbf{90.4} \textbf{\color{Green}{$\uparrow$7.5}} &
        \textbf{93.4} \textbf{\color{Green}{$\uparrow$6.0}}
        \\
        \rowcolor{LightGreen}
        USL-T (Ours) & 
        \textbf{66.0} \textbf{\color{Green}{$\uparrow$22.6}} & 
        \textbf{71.5} \textbf{\color{Green}{$\uparrow$15.6}} &
        \textbf{76.1} \textbf{\color{Green}{$\uparrow$15.3}} &
        \textbf{93.5} \textbf{\color{Green}{$\uparrow$10.6}} &
        \textbf{93.0} \textbf{\color{Green}{$\uparrow$5.6}} \\
        \shline
        \end{tabular}
    }
    \caption{USL/USL-T is a universal method that brings significant accuracy gains to various SSL methods. 
    % including MixMatch, SimCLRv2-CLD, FixMatch and CoMatch. 
    Experiments are conducted on CIFAR-10 with 40 labels. $\dagger$: practically infeasible, as it assumes perfectly balanced labeled instances.
    % practically infeasible stratified sampling which assumes perfectly balanced labeled instances.
    }
    \label{tab:cifar-10-additional}
    \end{table}
}

% Table 4: CIFAR-100
\def\tabCifarOneHundred#1{
    %\begin{table}[#1]
    \centering
    \tablestyle{3pt}{0.9}
    \normalsize
    \tablefontsmall{
        \begin{tabular}{l|l|l}
        \shline
        \textbf{CIFAR-100} & S.v2-CLD Acc & FixMatch Acc \\
        \shline
        Random Selection & 26.5 & 48.7 \\
        \textcolor{gray}{Stratified Selection$^\dagger$} & \textcolor{gray}{30.6} & \textcolor{gray}{51.2} \\
        \hline
        \rowcolor{LightGreen}
        USL (Ours) &\bf  33.0 \textbf{\color{Green}{$\uparrow$6.5}} &\bf 55.1 \textbf{\color{Green}{$\uparrow$6.4}} \\
        \rowcolor{LightGreen}
        USL-T (Ours) &\bf  36.9 \textbf{\color{Green}{$\uparrow$10.4}} &\bf 55.7 \textbf{\color{Green}{$\uparrow$7.0}} \\
        \shline
        \end{tabular}
    }
    \caption{By selecting informative samples to label, USL and USL-T greatly improve performance of SSL methods on CIFAR-100 with 400 labels.
    $\dagger$: practically infeasible, as it assumes perfectly balanced labeled instances.
    }
    \label{tab:cifar-100-ssl}
}

% Table 5: ImageNet-100
\def\tabImageNetOneHundred#1{
    %\begin{table}[#1]
    \centering
    \tablestyle{5pt}{0.9}
    \normalsize
    \tablefontsmall{
    \begin{tabular}{l|l}
    \shline
    % & \multicolumn{4}{c}{Accuracy (\%)} \\
    \textbf{ImageNet-100} & SimCLRv2 Acc \\
    \shline
    Random & 62.2 \\
    \textcolor{gray}{Stratified$^\dagger$} & \textcolor{gray}{65.1} \\
    \hline
    \rowcolor{LightGreen}
    USL (Ours) &\bf  67.5 \textbf{\color{Green}{$\uparrow$5.3}} \\
    \rowcolor{LightGreen}
    USL-T (Ours) &\bf  68.3 \textbf{\color{Green}{$\uparrow$6.1}} \\
    \shline
    \end{tabular}
    }
    \caption{USL and USL-T scale well to high dimensional image inputs with many classes on ImageNet-100 \cite{van2020scan}. $\dagger$: practically infeasible.
    }
    \label{tab:imagenet-100-ssl}
    %\end{table}
}

% Table 4 and 5:
\def\tabCifarOneHundredImageNetOneHundred#1{
\begin{table}[#1]
\begin{minipage}[t]{0.56\textwidth}
    \centering
    \tabCifarOneHundred{t}
\end{minipage}
\hfill
\begin{minipage}[t]{0.4\textwidth}
    \centering
    \tabImageNetOneHundred{t}
\end{minipage}
\end{table}
}

% Table 6: ImageNet
\def\tabImageNet#1{
    \begin{table*}[#1]
    \centering
    \tablestyle{11pt}{0.9}
    \normalsize
    \tablefontsmall{
    \begin{tabular}{l|l|l|l|l}
    \shline
    & \multicolumn{2}{c|}{SimCLRv2} & \multicolumn{2}{c}{FixMatch} \\
    \textbf{ImageNet-1k} &\multicolumn{1}{c}{1\%} & 0.20\%  &\multicolumn{1}{c}{1\%} & 0.20\% \\
    \shline
    % \textit{Supervised learning:}    &       &       \\
    Random                 & 49.7  & 33.2   & 58.8  & 34.3  \\
    \textcolor{gray}{Stratified$^\dagger$} & \textcolor{gray}{52.0}  & \textcolor{gray}{36.4} & \textcolor{gray}{60.9$^*$} & \textcolor{gray}{41.1} \\ 
    \hline
    \rowcolor{LightGreen}
    USL-MoCo (Ours)   & 51.5 \color{Green}{$\uparrow$1.8} & 39.8 \color{Green}{$\uparrow$6.6} 
        & 61.6  \color{Green}{$\uparrow$2.8} & \textbf{48.6 \color{Green}{$\uparrow$14.3}}  \\
    \rowcolor{LightGreen}
    USL-CLIP (Ours)   & \bf 52.6 \color{Green}{$\uparrow$2.9} & \bf 40.4 \color{Green}{$\uparrow$7.2} 
        & \bf 62.2 \color{Green}{$\uparrow$3.4} & 47.5 \color{Green}{$\uparrow$13.2} \\ \hline
    \shline
    \end{tabular}
    }
    \caption{Our proposed methods scale well on large-scale dataset ImageNet \cite{ILSVRC15}.
    % with SimCLRv2 \cite{chen2020big} and FixMatch \cite{sohn2020fixmatch, cai2021exponential}. 
    $^*$: reported in \cite{cai2021exponential}. 
    USL-MoCo and USL-CLIP use MoCov2 features and CLIP features, respectively, to perform selective labeling. $\dagger$: not a fair comparison. 
    % See comparison with SSAL in Fig. \ref{fig:imagenet_ssal_comp}.
    }
    \label{tab:imagenet}
    \end{table*}
}

% Table 7: MedMNIST: USL's strength without any training on target domain: MedMNIST
\def\tabMedMNIST#1{
    %\begin{table}[#1]
    \centering
    \tablestyle{2pt}{1.0}
    \normalsize
    \tablefontsmall{
    \begin{tabular}{l|c}
    \shline
    Selection Method & Accuracy \\
    \shline
    Random           & 77.17 $\pm$ 6.98 \\
    \textcolor{gray}{Stratified$\dagger$} & 
                   \textcolor{gray}{80.46 $\pm$ 7.88} \\
    \hline
    \rowcolor{LightGreen}
    USL (Ours)    & \textbf{88.06} $\pm$ 1.41 {\color{Green}\bf$\uparrow$10.89} \\
    \shline
    \end{tabular}
    }
    \caption{USL shows remarkable generalizability across domains without \textit{any} pre-training or fine-tuning on the target domain in BloodMNIST \cite{yang2021medmnist}. Annotated samples are chosen by a self-supervised CLD model trained on CIFAR-10 and \textit{never exposed to medical images}. We adopt the same hyperparams as FixMatch on CIFAR-10, except that we train only for 64 epochs.
    $\dagger$:~outperforming stratified with less information.
    }
    \label{tab:medmnist}
    %\end{table}
}

% Table 8: Loading the backbone into CIFAR-10 and CIFAR-100
\def\tabLoadingCIFAR#1{
    \centering
    \tablestyle{3pt}{0.95}
    \normalsize
    \tablefontsmall{
    \begin{tabular}{l|l|l}
    \shline
    Weights & Selection Method & Accuracy \\
    \shline
    SimCLR\cite{chen2020simple}                   & Random                & 55.9 \\ % & 23.9 \\
    % \textcolor{gray}{SimCLR\cite{chen2020simple}} & \textcolor{gray}{Stratified$\dagger$} & \textcolor{gray}{69.8} \\ % & 32.7 \\
    SimCLR\cite{chen2020simple}                   & USL-T (Ours)           & 71.5 \\ % & 34.2 \\
    CLD\cite{wang2021unsupervised}                & USL-T (Ours)           & 77.2 \\ % & 34.2 \\
    \hline
    \rowcolor{LightGreen}
    USL-T (Ours)                                   & USL-T (Ours)           &\bf 85.4 {\color{Green}\bf$\uparrow$8.2} \\
    \shline
    \end{tabular}
    }
    \caption{The backbone weights learned as a by-product in USL-T capture more semantic information, thereby working as a good initialization. 
    }
    \label{tab:loading_cifar}
}

% \def\tabMedMNISTLoadingCIFAR#1{
%     \begin{table}[#1]
%     \begin{minipage}[t]{0.55\textwidth}
%         \centering
%         \tabMedMNIST{t}
%         % \captionof{table}{Table}
%     \end{minipage}
%     \hfill
%     \begin{minipage}[t]{0.43\textwidth}
%         \centering
%         \tabLoadingCIFAR{t}
%         % \captionof{figure}{Diagram.}
%     \end{minipage}
%     \end{table}
% }

% \def\tabHyperparamHorizontal#1{
%     \begin{table}[#1]
%     \begin{minipage}[t]{0.6\textwidth}
%         \centering
%         \tablestyle{2.5pt}{0.9}
%         \normalsize
%         \tablefontsmall{
%         \begin{tabular}{l|c|c}
%         \shline
%         Hyperparam & CIFAR-10/-100 & ImageNet-100/-1k \\
%         \shline
%         % Symbols are not finalized
%         Neighborhood Size $k$ & \multicolumn{1}{c|}{20} & \multicolumn{1}{c}{50} \\
%         \hline
%         Adjustment Factor $\alpha$ & \multicolumn{2}{c}{5} \\
%         Momentum $\mu$ & \multicolumn{2}{c}{0.5} \\
%         Loss Term Weight $\lambda$ & \multicolumn{2}{c}{5} \\
%         Filter Threshold $\tau$ & \multicolumn{2}{c}{0.95} \\
%         Temperature $t$ & \multicolumn{2}{c}{0.25} \\
%         \shline
%         \end{tabular}
%         }
%     \end{minipage}\hfill
%     \begin{minipage}[t]{0.34\textwidth}
%         \caption{Our hyperparams generalize across domains, image sizes, and number of classes in USL-T. Hyperparams for USL are listed in appendix.}
%         \label{tab:hyperparameters}
%     \end{minipage}
%     \end{table}
% }

% Table 9: Hyperparameters
\def\tabHyperparamVertical#1{
    \centering
    \tablestyle{3pt}{1.0}
    \normalsize
    \tablefontsmall{
        \begin{tabular}{l|c|c}
        \shline
                   & CIFAR- & ImageNet- \\
        Hyperparam & 10/100 & 100/1k \\
        \shline
        Adjustment Factor $\alpha$ & \multicolumn{1}{c|}{5} & \multicolumn{1}{c}{2.5} \\
        Temperature $t$ & \multicolumn{1}{c|}{0.25} & \multicolumn{1}{c}{0.5} \\
        Loss Term Weight $\lambda$ & \multicolumn{1}{c|}{5} & \multicolumn{1}{c}{0.5} \\
        \hline
        Neighborhood Size $k$ & \multicolumn{2}{c}{20} \\
        Momentum $\mu$ & \multicolumn{2}{c}{0.5} \\
        % Filter Threshold $\tau$ & \multicolumn{2}{c}{0.95} \\
        \shline
        \end{tabular}
    }
    \caption{Hyperparams for USL-T. Hyperparams for USL are in appendix.
    }
    \label{tab:hyperparameters}
    % \end{table}
}

\def\tabLoadingCIFARHyperparam#1{
\begin{table}[#1]
\begin{minipage}[t!]{0.5\textwidth}
    \centering
    \tabLoadingCIFAR{t}
\end{minipage}
\hfill
\begin{minipage}[t!]{0.48\textwidth}
    \centering
    \tabHyperparamVertical{t}
\end{minipage}
\end{table}
}

We evaluate our USL and USL-T by integrating them into both pseudo-label based SSL methods (FixMatch \cite{sohn2020fixmatch}, MixMatch \cite{berthelot2019mixmatch}, or CoMatch\cite{li2021comatch}) and transfer-based SSL methods (SimCLRv2 and SimCLRv2-CLD\cite{chen2020big, wang2021unsupervised}).
% including pseudo-label SSL method FixMatch \cite{sohn2020fixmatch}, MixMatch \cite{berthelot2019mixmatch}, current SOTA method CoMatch\cite{li2021comatch} and transfer-learning SSL method SimCLRv2\cite{chen2020big} and SimCLRv2-CLD\cite{chen2020big, wang2021unsupervised}. 
We also compare against various AL/SSAL methods. Lastly, we show several intriguing properties of USL/USL-T such as generalizability.
% including training-free cross-domain generalization in USL and improved performance with learned backbone weights in USL-T.

%%%%%%%%%%%%%%%%%%%%%%%%%%%%%%%%%%%%%%%%%%%%%%%%%%%%%%%%%%%%%%%%%%%%%%%%%%%%%%%%%%%%%%%%%%%%%%%%%%%%
\tabfigALSSAL{t!}

\subsection{CIFAR-10}

\tabfigCifarSSLBalance{t}

We compare against mainstream SSL methods such as FixMatch\cite{sohn2020fixmatch} and SimCLRv2-CLD\cite{chen2020big,wang2021unsupervised} on extremely low-label settings to demonstrate our superior label efficiency. The labeling budget is 40 samples in total unless otherwise stated.
% including random selection baseline, AL, SSAL, and our methods. 
Note that the self-supervised models used for instance selection are trained on CIFAR-10 from scratch entirely  \textit{without external data}. The SSL part, including backbone and hyperparameters, is untouched. 
See appendix for details.

\paragraph{Comparison with AL and SSAL.}  Table \ref{tab:cifar-10-al-ssal} compares ours against various recent AL/SSAL methods in terms of sample efficiency and accuracy. AL methods operate at a much larger labeling budget than ours ($187\x$ more), because they rely only on labeled samples to learn both features and classification. SSAL methods make use of unlabeled samples and have higher label efficiency.  However, we achieve much higher accuracy with fewer labels requested. % Note that MMA also attempts to use of $K$-Means, but it does not contribute to significant improvement.

%\paragraph{Comparison With Various Selective Labeling Techniques in SSL Setting.} 
To tease apart whether our performance gains come from SSL or selective labeling, we tune recent AL/SSAL methods with their public implementations and run experiments with the same total budget, i.e. 40 samples in a 20 random + 20 selected setting. We then apply AL/SSAL selections to the same SSL for a fair comparison (Table \ref{tab:cifar-10-ssl}).

% We replace MixMatch in MMA with more recent FixMatch for fairer comparison (denoted as MMA$^+$). 
While AL performs better than random selection in SimCLRv2-CLD, its advantage saturates on FixMatch.  Since AL relies on labeled samples to learn the right features, with 20 random samples, it is very difficult to learn meaningful features for selection. Instead, AL could only learn a very coarse selection criterion and hence limited gains.

SSAL methods have greater gains on SimCLRv2-CLD. However, since SSAL still depends on initial random selections which seldom cover all 10 classes, these methods do not have an accurate knowledge of the full dataset in the low-label setting, where many rounds of queries are infeasible.  That is, there is a serious trade-off in the low-label regime: Allowing more samples (e.g., 30) in the initial random selection for better coverage means less annotation budget for AL/SSAL selection (e.g., 10).  Such a dilemma manifests itself in the imbalanced selection in Fig.~\ref{fig:balance} and the poor performance on FixMatch.

\noindent{\textbf{USL/USL-T as a Universal Method.}} In addition to mainstream SSL, we also use SimCLRv2, MixMatch\cite{berthelot2019mixmatch}, and SOTA CoMatch\cite{li2021comatch} for a comprehensive evaluation in Table~\ref{tab:cifar-10-additional}. We observe significant accuracy gains on all of them.

\subsection{CIFAR-100}
On CIFAR-100, we keep hyperparameters the same as the ones for CIFAR-10, except that we change the budget level to $400$  to have 4 labels per class on average. Although we may benefit more from hyperparameter tuning, we already show \textit{consistent gains} over other selection methods (Table \ref{tab:cifar-100-ssl}). % We observe $2.5\times$ improvement in terms of label efficiency in Fig.~\ref{fig:cifar100_ssal_comp}, compared with various SSAL methods.

\tabCifarAllMethods{t!}

% while MMA$^+$'s limited improvement in low-label scenario is consistent with the marginal improvement compared to random selection ($0.26\%$) with 4000 labels in MMA paper \cite{song2019combining}, 

%%%%%%%%%%%%%%%%%%%%%%%%%%%%%%%%%%%%%%%%%%%%%%%%%%%%%%%%%%%%%%%%%%%%%%%%%%%%%%%%%%%%%%%%%%%%%%%%%%%%

\subsection{ImageNet-100 and ImageNet-1k}

To demonstrate our effectiveness on large-scale datasets, we benchmark on 100 random classes of ImageNet \cite{van2020scan} and the full ImageNet \cite{ILSVRC15}.% We present two settings: (1) ImageNet-100 with 100 randomly sampled classes from ImageNet \cite{van2020scan}, (2) Full ImageNet-1K.

\paragraph{ImageNet-100.} On SimCLRv2 with a budget of 400 labels in total, we outperform baselines by 6.1\% in this extremely low-label setting (Table \ref{tab:imagenet-100-ssl}).
% with high-resolution inputs (Table \ref{tab:imagenet-100-ssl}).

\tabCifarOneHundredImageNetOneHundred{!t}
\tabImageNet{t!}

\paragraph{ImageNet-1k: Setup.} We experiment on SimCLRv2 and FixMatch with 1\% ($12, 820$ labels) and 0.2\% ($2, 911$ labels) labeled data. We also design a variant of our method that utilizes features provided by CLIP \cite{radford2021learning}.  CLIP is trained on uncurated internet-crawled data in a wide range of domains.  Following \cite{cai2021exponential}, we initialize FixMatch parameters with MoCov2. See appendix for more details.

% Add figure
\paragraph{ImageNet-1k: Comparing With AL/SSAL Methods.} As most AL/SSAL methods in Table \ref{tab:cifar-10-al-ssal} do not scale to ImageNet, we compare our USL with SSAL methods specifically designed for ImageNet-scale settings \cite{emam2021active}.  Fig.~\ref{fig:imagenet_ssal_comp}  shows our $8\times$ improvement in terms of label efficiency. % As summarized in Table \ref{tab:imagenet}, our method provides up to 3.4\% gains in the SSL setting. What is more interesting is the 0.2\% setting, where our method leads to an improvement up to 14.3\% in SSL.
Table \ref{tab:imagenet} shows that our approach provides up to 14.3\% (3.4\%) gains in the 0.2\% (1\%) SSL setting.

\paragraph{ImageNet-1k: USL-CLIP.} Table \ref{tab:imagenet} shows samples selected according to both MoCov2 and CLIP features boost SSL performance.  USL-MoCo performs 1.1\% better than USL-CLIP in the FixMatch setting.  We hypothesize that it is, in part, due to a mismatch between parameter initialization (MoCov2) and the feature space used for the sampling process (CLIP).
However, for 1\% case, USL-CLIP performs 0.6\% better than USL-MoCo, showing a slight advantage of a model trained with sufficient general knowledge and explicit semantics. 
% More discussions and experiments in appendix.
% With SimCLRv2, USL-CLIP consistently outperforms USL-MoCo.
% We encourage readers to view discussions and additional experiments in appendix.

% \paragraph{Main results.} 

\iffalse
As summarized in Table \ref{tab:imagenet}, samples selected from both MoCo and CLIP features boost SSL performance. In the 1\% case, our method provides up to 3.4\% gains in the SSL setting. What is more interesting is the 0.2\% setting, where our method leads to an improvement up to 14.3\% in SSL. 
We found that USL-MoCo performs 1.1\% better than USL-CLIP in the FixMatch setting. This is, in part, due to a mismatch between parameter initialization (MoCo) and the feature space used for the sampling process (CLIP).
However, for 1\% case, we find that USL-CLIP performs 0.6\% better than USL-MoCo, demonstrating the benefits of using a model with sufficient general knowledge and explicit semantic information. With SimCLRv2, USL-CLIP consistently outperforms USL-MoCo. We encourage readers to view discussions and additional experiments in appendix.
\fi

%%%%%%%%%%%%%%%%%%%%%%%%%%%%%%%%%%%%%%%%%%%%%%%%%%%%%%%%%%%%%%%%%%%%%%%%%%%%%%%%%%%%%%%%%%%%%%%%%%%%
% \tabMedMNIST{t!}

\subsection{Strong Generalizability}
\label{generalizability}
% Because USL is training free after image features are obtained, we demonstrate the generalizability of our method with a variety of image features on several domains.

% We demonstrate the generalizability of USL with a variety of features and domains.

\paragraph{Cross-dataset Generalizability with CLIP.} 
Since CLIP does not use ImageNet samples in training and the downstream SSL task is not exposed to the CLIP model either, USL-CLIP's result shows strong cross-dataset generalizability in Table \ref{tab:imagenet}. It means that: \textbf{1)} When a new dataset is collected, we could use a general multi-modal model to skip self-supervised pretraining; \textbf{2)} Unlike AL where sample selection is strictly coupled with model training, our annotated instances work \textit{universally} rather than with only the model used to select them.

\paragraph{Cross-domain Generalizability.} 
Such generalizability also holds \textit{across domains}. We use a CLD model trained on CIFAR-10 to select 40 labeled instances in medical imaging dataset BloodMNIST \cite{yang2021medmnist}. Although our model has \textit{not} been trained on any medical images, our model with FixMatch performs 10.9\% (7.6\%) better than random (stratified) sampling.  See appendix for more details.

% \paragraph{Cross-domain Generalizability.} To further analyze whether such generalizability also holds \textit{across domains}, we use the exact same CLD model pretrained on CIFAR-10 to select samples in the BloodMNIST dataset of the MedMNISTv2 collection \cite{yang2021medmnist} for annotation on 3 different folds and 40 labeled instances (5 labels/class on average). BloodMNIST contains about 18k blood cell images under microscope in 8 classes, which is drastically different from CIFAR-10's domain, but our model with FixMatch performs 10.89\% better than random sampling and 7.6\% better than stratified sampling, as in Table \ref{tab:medmnist}, further illustrating our generalizability.

\subsection{USL-T for Representation Learning}
\label{uslt_representation_learning}
Our USL-T updates feature backbone weights during selective labeling. The trained weights are not used as a model initializer in the downstream SSL experiments for fair comparisons.  However, we discover surprising generalizability that greatly exceeds  self-supervised learning models under the SimCLRv2 setting.
Specifically, we compare the performance of classifiers that are initialized with various model weights and are optimized on samples selected by different methods. Table~\ref{tab:loading_cifar} shows that, even with these strong baselines, initializing the model with our USL-T weights surpasses baselines by 8.2\%.
%while being exposed to the same amount of information.

\iffalse
Specifically, we substitute the backbone weights in the SimCLRv2 and present results in Table~\ref{tab:loading_cifar}. We benchmark on SimCLR since USL-T on CIFAR-10 starts from SimCLR weights, and we also benchmark on CLD weights that are found to perform better under low-label setting. Even with these strong baselines, our weights outperform SimCLR and CLD by 12.3\% and 8.2\%, respectively, while being exposed to same amount of information in SimCLR and CLD.
\fi

\tabLoadingCIFARHyperparam{t}

%%%%%%%%%%%%%%%%%%%%%%%%%%%%%%%%%%%%%%%%%%%%%%%%%%%%%%%%%%%%%%%%%%%%%%%%%%%%%%%%%%%%%%%%%%%%%%%%%%%%
%\figAL{ht!}
%\figSSAL{ht!}

% \figALSSAL{ht!}
% \figAblationMetrics{h!}
% \tabAblationCifar{ht}
% \figAblationCifar{t!}
% \figMetrics{t!}

%%%%%%%%%%%%%%%%%%%%%%%%%%%%%%%%%%%%%%%%%%%%%%%%%%%%%%%%%%%%%%%%%%%%%%%%%%%%%%%%%%%%%%%%%%%%%%%%%%%%

% \tabHyperparam{t!}

\subsection{Hyperparameters and Run Time}
\label{ablation}

% \paragraph{Hyperparameters and Run Time.} 
Table \ref{tab:hyperparameters} shows that our hyperparameters generalize within small-scale and large-scale datasets. 
% Similarly, we use the same hyperparams for all small-scale datasets (CIFAR-10/100 and MedMNIST) and the same for all large-scale datasets (ImageNet100/1000). 
% See appendix for the table for USL.
Our computational overhead is negligible. On ImageNet, we only introduce about 1 GPU hour for selective labeling, as opposed to 2300 GPU hours for the subsequent FixMatch pipeline.  See appendix for more analysis, including formulations and visualizations.

%% file: 5summary.tex
\section{Summary}

\label{summary}
% Existing SSL methods are model-centric, focusing on models and algorithms that integrate both labeled and unlabeled data.
% Instead, our DC-SSL is the first work that turns to labeled data selection in SSL in an unsupervised way.  By simply optimizing what is fed into a model, we demonstrate significant gains on annotation efficiency, model stability and accuracy on all experimented benchmarks.  Our work is also in a stark contrast to supervised data selection for active learning.  Our completely unsupervised data selection can be easily extended to other weakly supervised learning settings.

Unlike existing SSL methods that focus on algorithms that better integrate labeled and unlabeled data, our selective-labeling is the first to focus on unsupervised data selection for labeling and enable more effective subsequent SSL.  By choosing a diverse representative set of instances for annotation, we show significant gains in annotation efficiency and downstream accuracy, with remarkable selection generalizability within and across domains. % We are committed to releasing a toolbox for fair benchmarks.
% , and we believe our work will facilitate research on future data-centric methods. 
% our completely unsupervised data selection can be easily extended to other weakly supervised learning settings

\paragraph{Acknowledgements.} 
The authors thank Alexei Efros and Trevor Darrell for helpful discussions and feedback on this work in their classes.

%\textbf{Limitations and Future Works:} 
%This is just the first step towards exploiting the informativeness of data for semi-supervised learning.Even though DC-SSL successfully enhances sample selection process, several problems still remain open: 1) To avoid overly complicating our methods, we leave more complex scenarios (e.g. highly imbalanced class-wise distribution and fine-grained semantics) to future work. 2) Self-supervised feature learning is not aware of label selection stage, and unifying them is worthy of exploration. % Adding awareness by merging them into one stage is worthy of exploration. % Hence, methods to unify them into a single end-to-end optimized framework and rank the observations based on the informativeness scores predicted by it are worthy of exploration. 

% The proposed method is a pure machine learning based sample selection method without using any deep learning models, however, whether it is possible to use the CNN models to realize the sample selection without using any human annotations?

\iffalse
\newpage
\subsection*{Reproducibility Statement}
In order to make our results reproducible, we listed all used parameters and changes we applied for training in the Section~\ref{appendix-setup-cifar} and Section~\ref{appendix-setup-imagenet} of Appendix and Section 4.1, 4.2 and 4.3 of Experiment, and cited the corresponding papers on which our method is build on which provide publicly available Github repositories themselves. We are committed to reproducible results reported in the paper and public code release.
\fi

%% file: 6appendix.tex
% \clearpage
\section{Appendix}
\setcounter{theorem}{0}
\subsection{Relationships Between Global Loss in USL-T and the K-Means Clustering Objective}
Intuitively, the global loss in our proposed USL-T performs deep clustering. Furthermore, a connection can be observed between minimizing global loss and performing a generalized form of K-Means clustering, which reduces to K-Means clustering with an additional regularization term when $\tau = 0$ and the feature space is fixed.

\begin{observation}
Assume that $\tau = 0$ and the feature space is fixed, minimizing $L_{\text{global}}$ optimizes the objective of $K$-Means clustering with a regularization term on the inter-cluster distance that encourages additional diversity.
\end{observation}
\begin{proof}
Recall that we have one-hot assignment $y(x)$ and soft assignment $\hat{y}(x)$  defined as:
\begin{align}
% \hat{y}_i(x) = \delta(i, \argmin_{k \in \{1, ..., C\}} s({f(x)}, {c_k}))
y_i(x) = 
\begin{cases}
    1,& \text{if } i = \arg\min_{k \in \{1, ..., C\}} s({f(x)}, {c_k})\\
    0,              & \text{otherwise}
\end{cases}
\end{align}
\begin{align}
\hat{y}_i(x)= \frac{e^{s({f(x)}, {c_i})}}{\sum_{j=1}^C e^{s({f(x)}, {c_j})}}
\end{align}
where $c_i \in \mathbb{R}^{d}, k \in \{1, ..., C\}$ are learnable centroids with feature dimension $d$, $s(\cdot, \cdot)$ is a function that quantifies the similarity between two points in a feature space $\mathbb{R}^{d}$ and $f(x) \in \mathbb{R}^{d}$ is a function that maps an input $x$ to a feature space, which is implemented by a CNN.

Then our global loss is defined as:
\begin{align}
L_{\text{global}}(\mathcal{X}) = \frac{1}{|\mathcal{X}|}\sum_{x \in \mathcal{X}} D_{\text{KL}}({y}(x) || \hat {y}(x)) F(\hat {y}(x))
\end{align}

When $\tau = 0$, the filtering function $F(\hat {y}(x)) = \mathbb{1}(\max( \hat {y}(x)) \geq \tau)$ has no effect.
% Since the filtering function $F(\hat {y}(x)) = \mathbb{1}(\max( \hat {y}(x)) \geq \tau)$ with confidence threshold $\tau$ is a binary function which excludes some samples in loss calculation, we could define $\mathcal{X}'$ to be the samples that pass the filtering threshold:
% \begin{align}
% \mathcal{X}' = \{x \in \mathcal{X} \mid F(\hat {y}(x)) = 1\}
% \end{align}
Then we can simplify our global loss and turn the loss into the following form:

\begin{align}
L_{\text{global}}(\mathcal{X}) = \frac{1}{|\mathcal{X}|}\sum_{x \in \mathcal{X}} D_{\text{KL}}({y}(x) || \hat {y}(x))
\end{align}

Since the feature space is fixed, i.e. does not change across the loss optimization, $f(x)$ remains constant, and the goal is to find the optimal centroids $\{c^*_i\}_{i=1}^C$ that minimizes the loss:

\begin{align}
\{c^*_i\}_{i=1}^C = \argmin_{\{c_i\}_{i=1}^C} L_{\text{global}}(\mathcal{X})
\end{align}

We then get:

\begin{align}
    \{c^*_i\}_{i=1}^C 
              &= \argmin_{\{c_i\}_{i=1}^C} L_{\text{global}}(\mathcal{X}) \\
              &= \argmin_{\{c_i\}_{i=1}^C} \frac{1}{|\mathcal{X}|}\sum_{x \in \mathcal{X}} D_{\text{KL}}({y}(x) || \hat {y}(x)) \\
              &= \argmin_{\{c_i\}_{i=1}^C} \sum_{x \in \mathcal{X}}\sum_{i=1}^C -{y}(x) \log \frac{\hat {y}(x)}{{y}(x)} \label{eq:1}
\end{align}

Since ${y}(x)$ is a one-hot vector, we can simplify \eqref{eq:1} further.
Define $M(x) = \argmin_k s(f(x), c_k)$ and $s(\cdot, \cdot) = -d(\cdot, \cdot)$ for some metric $d(\cdot, \cdot)$,
\begin{align}
    \{c^*_i\}_{i=1}^C 
              &= \argmin_{\{c_i\}_{i=1}^C} \sum_{x \in \mathcal{X}} - \log \hat {y}(x)_{M(x)} \\
              &= \argmin_{\{c_i\}_{i=1}^C} \sum_{x \in \mathcal{X}} - \log \frac{e^{s(f(x), c_{M(x)})}}{\sum_{k=1}^C e^{s(f(x), c_k)}} \\
              &= \argmin_{\{c_i\}_{i=1}^C} \sum_{x \in \mathcal{X}} - \log e^{s(f(x), c_{M(x)})} + \log{\sum_{k=1}^C e^{s(f(x), c_k)}} \\
              &= \argmin_{\{c_i\}_{i=1}^C} \sum_{x \in \mathcal{X}} - \log e^{-d(f(x), c_{M(x)})} + \log{\sum_{k=1}^C e^{-d(f(x), c_k)}} \\
              &= \argmin_{\{c_i\}_{i=1}^C} \sum_{x \in \mathcal{X}} d(f(x), c_{M(x)}) + \log{\sum_{k=1}^C e^{-d(f(x), c_k)}}
\end{align}

If we let $d(\cdot, \cdot)$ be squared L2 distance, the expression can be decomposed into the sum of a square L2 distance with an regularization term:

\begin{align}
\{c^*_i\}_{i=1}^C &= \argmin_{\{c_i\}_{i=1}^C} ~(\text{Main objective} + \text{Reg})
\end{align} 
where
\begin{align}
\text{Main objective} &= \sum_{x \in \mathcal{X}} ||x - c_{M(x)}||^2 \label{eq:main-term} \\
\text{Reg} &= \log{\sum_{k=1}^C e^{-d(f(x), c_k)}} = \log{\sum_{k=1}^C e^{-||f(x) - c_k||^2}} \label{eq:regularization-term}
\end{align}

Minimizing the regularization term is equivalent to maximizing the sample's distance to all clusters $d(f(x), c_k)$, $\forall k \in \{1, ..., C\}$. This pushes apart different clusters and contributes to the diversity between clusters:

For $k \neq M(x)$, there is only force from the regularization term, which pushes apart a sample and other clusters that it does not belong to.

For $k = M(x)$, there are two forces: one from the main objective (\eqref{eq:main-term}) and one from the regularization term (\eqref{eq:regularization-term}). The regularization term pushes the sample away from its assigned cluster, i.e. the regularization term maximizes $d(x, c_k)$ also for $k = M(x)$, while the main objective minimizes $d(x, c_{M(x)})$.

We can quantify the net effect for $k = M(x)$ scenario. The gradient of the regularization term w.r.t $d(x, c_{M(x)})$ is $-\hat {y}(x)_{M(x)}$, and the gradient from the main objective to $d(x, c_{M(x)})$ is always 1. As $\hat {y}(x)$ is a probability distribution, $0 \leq \hat {y}(x)_{M(x)} \leq 1$. Therefore, the net effect is still minimizing $d(x, c_{M(x)})$, i.e. attracting $x$ to its cluster center $c_{M(x)}$ and $c_{M(x)}$ to $x$.

Therefore, \eqref{eq:regularization-term} is a regularization term aiming for additional diversity. 

Now we consider the objective without the regularization term, and we define the centroids without the regularization term $\{c'_i\}_{i=1}^C$ as:
\begin{align}
\{c'_i\}_{i=1}^C &= \argmin_{\{c_i\}_{i=1}^C} \sum_{x \in \mathcal{X}} ||x - c_{M(x)}||^2\label{eq:c_prime}
\end{align}

Since there is no interdependence between $c_i$ and $c_j$, where $i \neq j$, we can define
\begin{align}
\mathcal{X}_{k}' = \{x \in \mathcal{X} \mid M(x) = k\}
\end{align}
and write \eqref{eq:c_prime} as

\begin{align}
\{c'_i\}_{i=1}^C &= \argmin_{\{c_i\}_{i=1}^C} \sum_{k=1}^C h_k\label{eq:c_prime2}\\
h_k &= \sum_{x \in \mathcal{X}_{k}'} ||x - c_k||^2
\end{align}

Then the solution to \eqref{eq:c_prime2} is to minimize the individual $h_k, \forall k \in \{1, 2, ..., C\}$, i.e. the sum of squared L2 distances between a cluster and the samples that belong to it.

Without loss of generalizability, we analyze $c_1'$,

\begin{align}
\{c'_i\}_{i=1}^C &= \argmin_{c_1} h_1 \\
                 &= \argmin_{c_1} \sum_{x \in \mathcal{X}_{1}'} ||x - c_1||^2 \label{eq:c1_expand}
\end{align}

The gradient of the objective in \eqref{eq:c1_expand} w.r.t $c_1$ is

\begin{align}
\nabla_{c_1}h_1 = -2\sum_{x \in \mathcal{X}_{1}'} x - c_1 \label{eq:c1_grad}
\end{align}

since the objective is convex, \eqref{eq:c1_grad} indicates the unique minimum is reached when

\begin{align}
\sum_{x \in \mathcal{X}_{1}'} x - c_1 &= 0 \\
c_1 &= \frac{1}{|\mathcal{X}_{1}'|}\sum_{x \in \mathcal{X}_{1}'} x
\end{align}

This means that $c_i'$ is the mean of all sample vectors that belong to cluster $i$. This indicates that \eqref{eq:c_prime} is equivalent to the objective of $K$-Means clustering, which aims to minimize the square L2 distance between a group of samples that are assigned to a specific cluster and the mean of this group of samples.

Therefore, $L_{\text{global}}(\mathcal{X})$ has same objective with $K$-Means clustering with an extra regularization term on maximizing the inter-cluster sample distances for cluster diversity.
\end{proof}

\subsection{Non-optimality of Two Types of Collapses}

\begin{observation}
Neither one-cluster nor even-distribution collapse is optimal to our local constraint, i.e. $P(z(x'), \bar z, t) \neq \hat y({x})$ for either collapse.
\end{observation}
\begin{proof}

Let $z(x) \in \mathbb{R}^{d}$ be the logits of $x$ and $\bar z \in \mathbb{R}^{d}$ be the moving average of the batch mean of $\sigma(z(x'))$, with $\sigma(\cdot)$ as the softmax function and $\mu$ as the momentum:
\begin{align}
z_k(x) &= s(f(x), c_k) = f(x)^\intercal c_k \\
\bar z &\leftarrow \mu (\frac{1}{n}\sum_{i=1}^n\sigma(z(x_i')))\!+\!(1\!-\!\mu)\bar z  \text{~~at each iteration}
\end{align}

Recall that we define our anti-collapsing function $P(z, \bar z, t)$ with two components, as:
\begin{align}
\hat{P}({z}, \bar{z}) &= {z}\!-\alpha\!\log \bar z \\
[P'(\hat{z}, t)]_i &= \frac{\exp(\hat{z}_i/t)}{\sum_j\exp(\hat{z}_j/t)} \\
P(z, \bar z, t) &= P'(\hat{P}(z, \bar{z}), t) \\
% [P(z, \bar z, t)]_i &=\frac{\exp(\hat{P}(z_i, \bar{z_i})/t)}{\sum_j\exp(\hat{P}(z_j, \bar{z_j})/t)}
\end{align}
where $\alpha$ is the adjustment factor and $t$ is the temperature.

Then the local loss is formulated as:
\begin{align}
l_{\text{local}}(x_i, x'_i) &= D_{\text{KL}}(P(z(x'_i), \bar z, t) || \hat y(x_i)) \\
L_{\text{local}}(\{{x_i}\}_{i=1}^n) &= \frac{1}{n} \sum_{i=1}^n l_{\text{local}}(x_i, x'_i)
\end{align}
where $x_i'$ is a randomly picked neighbor from the $k$ nearest neighbors of $x_i$.

According to Jensen's inequality, KL divergence $D_{\text{KL}}(p || q)$ only achieves optimality, with gradient norm 0, when $p = q$. To prove a solution is not optimal for $l_{\text{local}}(x, x')$, we only need to prove $P(z(x'), \bar z, t) \neq \hat y(x)$.

For one-cluster collapse, where the neural network assigns all samples to the same cluster with high confidence, both $\hat y(x')$ and $\hat y(x)$ are very close to a one-hot distribution, with $\sigma(z(x_i')) \approx \frac{1}{n}\sum_{i=1}^n\sigma(z(x_i')), \forall i \in \{1, 2, ... , n\}$.

Assuming that we have already taken enough iterations for the moving average to catch up in this collapsing situation, that the difference between $\bar z$ and $\frac{1}{n}\sum_{i=1}^n\sigma(z_i)$ is negligible: $\bar z \approx \frac{1}{n}\sum_{i=1}^n\sigma(z(x_i'))$.
We have:
\begin{align}
\hat{P}(z, \bar{z}) &= {z}\!-\alpha\!\log \bar z \\
                &\approx {z}\!-\alpha\!\log \sigma(z) \\
                &= c\vec{1}_d
\end{align}
where $c$ is a constant and $\vec{1}_d \in \mathbb{R}^d$ is a vector of 1 when $\alpha=1$. With $\alpha > 1$, $\hat{P}(z, \bar{z})$ has an even stronger adjustment effect that pushes the target distribution even farther than uniform distribution, which is found to be beneficial in our circumstances with an optimizer that is using momentum for faster convergence.

% $c = \log\sum_{i=1}^C e^{[z]_i}$

Note that $\hat y(x)$ is a distribution close to one-hot, as above, we have:
\begin{align}
P(z, \bar z, t) &= P'(\hat{P}(z, \bar{z}), t) \\
                &\approx P'(c\vec{1}_d, t) \\
                &= \frac{1}{C}\vec{1}_d \\
                &\neq \hat y(x)
\end{align}
where C is the number of clusters.

Therefore, our loss will drive the distribution back to one that is less extreme, i.e. close to uniform distribution, and thus one-cluster collapse is not an optimum for $l_{\text{local}}(x)$.

Now we consider even distribution collapse, where samples are assigned to a distribution close to uniform distribution, with mean distribution of all samples being uniform, i.e. the mapping function $f(x)$ assigns similar logits and thus similar distributions to all clusters with small variations that are drawn from a distribution with zero-mean. Note that exact uniform distribution, where the variation exactly equals to 0, is hardly achieved in the optimization process and thus is not a concern for us.

Consider one such sample $x$ with neighbors $x_i'$ and predicted logits $z$ from the batch. Here we also assume that the exponential moving average catches up, so that we have: $\bar z \approx \frac{1}{C}\vec{1}_d$.
In this case, $\hat{P}(z, \bar{z})$:
\begin{align}
\hat{P}(z, \bar{z}) &= {z}\!-\alpha\!\log \bar z \\
                &\approx {z}\!-\alpha\!\log \frac{1}{C}\vec{1}_d \\
                &= {z}\!-\alpha\!\log \frac{1}{C}
\end{align}

Note that $P'(\hat z(x), t)$ is by design invariant to an additive constant on the $\hat z(x)$:
\begin{align}
    [P'(\hat{z}(x) + c, t)]_k &= \frac{\exp((\hat{z}_k(x) + c)/t)}{\sum_j\exp((\hat{z}_j(x) + c)/t)} \\
    &= \frac{\exp(\hat{z}_k(x)/t) \exp(c/t)}{\sum_j\exp(\hat{z}_j(x)/t) \exp(c/t)} \\
    &= \frac{\exp(\hat{z}_k(x)/t)}{\sum_j\exp(\hat{z}_j(x)/t)} \\
    &= [P'(\hat{z}(x), t)]_k
\end{align}

Then we consider the net effect of $\hat{P}$ and $P'$:
\begin{align}
I(z(x'), \bar z, t) &= P'(\hat{P}(z(x'), \bar{z}), t) \\
                &\approx P'(z(x')\!-\alpha\!\log \frac{1}{C}, t) \\
                &= P'(z(x'), t) \\
                &\neq \hat{y}(x)
\end{align}

The last step comes from the fact that $z(x')$ contains some variations and is not a uniform distribution. In this case, $P'(z(x'), t)$ will enlarge the dimension of z which has maximum value and make other dimension smaller in the output probability,
forcing the softmax distribution to be spikier during training.
Therefore, $I(z(x'), \bar z, t)$ will have a distribution that makes the variation more significant, driving the distribution out of mean cluster collapse.

% However, note that \textit{confident but random assignment}, i.e. $f(x)$ assigns each sample $x$ confidently to a random cluster, with different samples assigned to a different cluster, is not within our concern, because with limited capacity as well as inductive bias in $f(x)$, it is very difficult for $f(x)$ to perform \textit{random} assignment as it requires distinguishing the property of each samples. Such assumption is the same as the one used in the literature of contrastive learning \cite{wu2018unsupervised,he2020momentum,chen2020simple, chen2020big,grill2020bootstrap,wang2021unsupervised}, that the network cannot perfectly distinguish samples. Such assumption always holds our situation as well.

% Since $\bar z$ is a moving average that lags behind $z$, $\bar z \neq z$, and thus $\hat{P}({z}, \bar{z})$ is still not uniform. In this case, the sharpening function $P'(\hat{z}, t)$ will enlarge the maximum entry in $\hat{P}({z}, \bar{z})$, which makes it far away from the uniform distribution at $\hat{y}(x)$, indicating that we are not at an optimum.

\end{proof}

%%%%%%%%%%%%%%%%%%%%%%%%%%%%%%%%%%%%%%%%%%%%%%%%%%%%%%%%%%%%%%%%%%%%%%%%%%%%%%%%%%%%%%%%%%%%%%%%%%%%
\def\algRegularization#1{
\begin{algorithm}[#1]
\caption{The iterative regularization algorithm}\label{alg:regularization}
\begin{algorithmic}
\REQUIRE
\STATE $\{U(V_i)| V_i \in \sV\}$: The unregularized utility for each vertex $V_i$
\STATE $\lambda$: weight for applying regularization
\STATE $m_{\text{reg}}$: momentum in exponential moving average \\
\STATE $l$: the number of iterations

\hspace*{-\algorithmicindent}\textbf{Procedure:} \\
\STATE {$\bar{\text{Reg}}(V_i, 0) \gets 0$, $\forall V_i \in \sV$}
\STATE $\hat{\sV}^0 \gets $ samples with largest $U(V_i)$ in each cluster
\FOR{$t=1$ \TO $l$}
    \FORALL{$V_i \in \sV$}
        \STATE $\text{Reg}(V_i, t) \gets \sum_{\hat{V}_j^{t-1}\not\in S_i} \frac{1}{\rVert V_i - \hat{V}_j^{t-1} \rVert^\alpha}$
        \STATE {$\bar{\text{Reg}}(V_i, t) \gets m_{\text{reg}}\cdot\bar{\text{Reg}}(V_i, t-1) + (1-m_{\text{reg}})\cdot\text{Reg}(V_i, t)$}
        \STATE $U'(V_i, t) \gets U(V_i) - \lambda\cdot\bar{\text{Reg}}(V_i, t)$
    \ENDFOR
    \STATE $\hat{\sV}^t \gets $ samples with largest $U'(V_i, t)$ in each cluster
\ENDFOR

\RETURN $\hat{\sV}^{l}$
\end{algorithmic}
\end{algorithm}
}

%%%%%%%%%%%%%%%%%%%%%%%%%%%%%%%%%%%%%%%%%%%%%%%%%%%%%%%%%%%%%%%%%%%%%%%%%%%%%%%%%%%%%%%%%%%%%%%%%%%%

\def\tabCifarBudget#1{
    \begin{table*}[#1]
    \centering
    \tablestyle{2pt}{1.0}
    \normalsize
    \tablefont{
    \begin{tabular}{l|l|l|l}
    \shline
    % & \multicolumn{4}{c}{Accuracy (\%)} \\
    Sampling Method & \multicolumn{1}{c}{40 labels} & \multicolumn{1}{c}{100 labels} & \multicolumn{1}{c}{250 labels} \\
    \shline
    Random & 60.8 & 73.7 & 79.4 \\
    \textcolor{gray}{Stratified$\dagger$} & \textcolor{gray}{66.5} &
     \textcolor{gray}{74.5} & \textcolor{gray}{80.4} \\
    \hline
    \rowcolor{LightGreen}
    USL (Ours) & \textbf{76.6} \textbf{\color{Green}{$\uparrow$ 15.8}} & \textbf{79.0} \textbf{\color{Green}{$\uparrow$ 5.3}} & \textbf{82.1} \textbf{\color{Green}{$\uparrow$ 2.7}} \\
    \rowcolor{LightGreen}
    USL-T (Ours) & \textbf{76.1} \textbf{\color{Green}{$\uparrow$ 15.3}} & - & - \\
    \shline
    \end{tabular}
    }
    \caption{CIFAR-10 experiments with transfer-learning based SSL method SimCLRv2-CLD \cite{chen2020big, wang2021unsupervised}, with the mean of 5 different folds and 2 runs in each fold. $\dagger$: Even though stratified selection uses more information and is not a fair comparison, we still outperform stratified selection.}
    \label{tab:cifar-10-budget}
    \end{table*}
}

\def\tabCifarBudgetFixMatch#1{
    \begin{table}[#1]
    \centering
    \tablestyle{3pt}{1}
    \normalsize
    \tablefont{
    \begin{tabular}{l|l|l|l}
    \shline
    & \multicolumn{3}{c}{Accuracy (\%)} \\
    Sample Selection &  \multicolumn{1}{c}{40 labels}   & \multicolumn{1}{c}{100 labels} & \multicolumn{1}{c}{250 labels} \\
    \shline
    Random           & 82.9 & 88.7 & 93.3 \\
    \textcolor{gray}{Stratified$\dagger$*} & \textcolor{gray}{88.6} & \textcolor{gray}{90.2} & \textcolor{gray}{94.9} \\ \hline
    \rowcolor{LightGreen}
    USL (Ours)    & \textbf{90.4} {\bf\color{Green}{$\uparrow$ 7.5}} & \textbf{93.2} {\bf\color{Green}{$\uparrow$ 4.5}} & \textbf{94.0} {\bf\color{Green}{$\uparrow$ 0.7}} \\
    \rowcolor{LightGreen}
    USL-T (Ours)    & \textbf{93.5} {\bf\color{Green}{$\uparrow$ 10.6}} & - & - \\
    \shline
    \end{tabular}
    }
    \caption{CIFAR-10 experiments with FixMatch \cite{sohn2020fixmatch}. 
    $\dagger$: Not a fair comparison with us because it assumes balanced labeled data available and leaks information about ground truth labels. *: results from \cite{sohn2020fixmatch}.}
    \label{tab:cifar-10-fixmatch-budget}
    \end{table}
}

\subsection{Additional Experiment Results}
\subsubsection{Varying Budgets.}
Table \ref{tab:cifar-10-budget} and \ref{tab:cifar-10-fixmatch-budget} indicate the accuracy with different budget levels on SimCLRv2-CLD and FixMatch, respectively. For SimCLRv2-CLD, our method consistently outperforms not only random selection but also stratified selection for \textit{all} the low-label settings. Our improvement is prominent especially when the number of selected samples is low. In 40 (250) labels case, we are able to achieve a 15.8\% (2.7\%) improvement. For FixMatch, we consistently outperform random baselines and even outperform stratified sampling, which makes use of ground truth labels of unlabeled data, in most of the settings.

%\tabCifarTransferLearning{h!}
%\tabCifarWide{h!}

\tabCifarBudget{h!}
\tabCifarBudgetFixMatch{h!}

\def\tabAdditionalUSLT#1{
    \begin{table}[#1]
    \centering
    \tablestyle{2pt}{1.0}
    \normalsize
    \tablefont{
    \begin{tabular}{l|c}
    \shline
    \textbf{ImageNet} & SimCLRv2 \\
    \shline
    Random           & 33.2 \\
    \textcolor{gray}{Stratified$\dagger$} & \textcolor{gray}{36.4}\\
    \hline
    \rowcolor{LightGreen}
    USL-MoCo (Ours)    & 39.8 {\color{Green}\bf$\uparrow$6.6} \\
    \rowcolor{LightGreen}
    USL-CLIP (Ours)    & 40.4 {\color{Green}\bf$\uparrow$7.2} \\
    \rowcolor{LightGreen}
    USL-T (Ours)  & \textbf{41.3} {\color{Green}\bf$\uparrow$8.1} \\
    \shline
    \end{tabular}
    }
    \caption{Additional USL-T experiments with SimCLRv2 \cite{chen2020big} on ImageNet. On ImageNet, USL-T also shows promising improvements, reaching a $6.6\%$ improvement when compared to baseline.
    $\dagger$: Although stratified selection utilizes ground truth, we still outperform it without using labeled information.}
    \label{tab:additional-uslt}
    \end{table}
}

\tabAdditionalUSLT{t!}

\subsubsection{USL-T on ImageNet.}
We also provide experimental results of USL-T on ImageNet in Table \ref{tab:additional-uslt}. As for the hyperparams for USL-T, we use the same hyperparams as shown in the hyperparam table in the main text. For ImageNet, to create a fair comparison, USL-T model is initialized with weights of MoCov2.

%%%%%%%%%%%%%%%%%%%%%%%%%%%%%%%%%%%%%%%%%%%%%%%%%%%%%%%%%%%%%%%%%%%%%%%%%%%%%%%%%%%%%%%%%%%%%%%%%%%%
\def\tabMedMNIST#1{
    \begin{table}[#1]
    \centering
    \tablestyle{20pt}{1.0}
    \normalsize
    \tablefont{
    \begin{tabular}{l|c}
    \shline
    Selection Method & Accuracy \\
    \shline
    Random           & 77.17 $\pm$ 6.98 \\
    \textcolor{gray}{Stratified$\dagger$} & 
                   \textcolor{gray}{80.46 $\pm$ 7.88} \\
    \hline
    \rowcolor{LightGreen}
    USL (Ours)    & \textbf{88.06} $\pm$ 1.41 {\color{Green}\bf$\uparrow$10.89} \\
    \shline
    \end{tabular}
    }
    \caption{USL shows remarkable generalizability across domains without \textit{any} pre-training or fine-tuning on the target domain in BloodMNIST \cite{yang2021medmnist}. Annotated samples are chosen by a self-supervised CLD model trained on CIFAR-10 and \textit{never exposed to medical images}. We adopt the same hyperparams as FixMatch on CIFAR-10, except that we train only for 64 epochs. Mean and standard deviation are taken over three runs.
    $\dagger$:~outperforming stratified with less information.
    }
    \label{tab:medmnist}
    \end{table}
}

\tabMedMNIST{t!}

\subsubsection{Cross-domain Generalizability on MedMNIST.}
We show USL's impressive generalizability in the main text through selective labeling with CLIP features in the ImageNet training set. Furthermore, to analyze whether USL's generalizability holds \textit{across domains}, we use the exact same CLD model pretrained on CIFAR-10 to select samples in the BloodMNIST dataset of the MedMNISTv2 collection \cite{yang2021medmnist}, which is a dataset in medical imaging domain. BloodMNIST contains about 18k blood cell images under microscope in 8 classes, which is drastically different from CIFAR-10's domain, but as shown in Table \ref{tab:medmnist}, our model with FixMatch performs $10.89\%$ better than random sampling and $7.60\%$ better than stratified sampling, further illustrating the possibility of a general sample selection model across image domains.

%%%%%%%%%%%%%%%%%%%%%%%%%%%%%%%%%%%%%%%%%%%%%%%%%%%%%%%%%%%%%%%%%%%%%%%%%%%%%%%%%%%%%%%%%%%%%%%%%%%%
\newpage

\subsection{CIFAR-10 Visualizations on Selected Samples}
\def\figCIFARVisualization#1{
\begin{figure}[#1]
\tablestyle{0pt}{0}
\begin{tabular}[t]{l}
\begin{subfigure}{1.0\linewidth}
    \centering
    \includegraphics[width=0.6\linewidth]{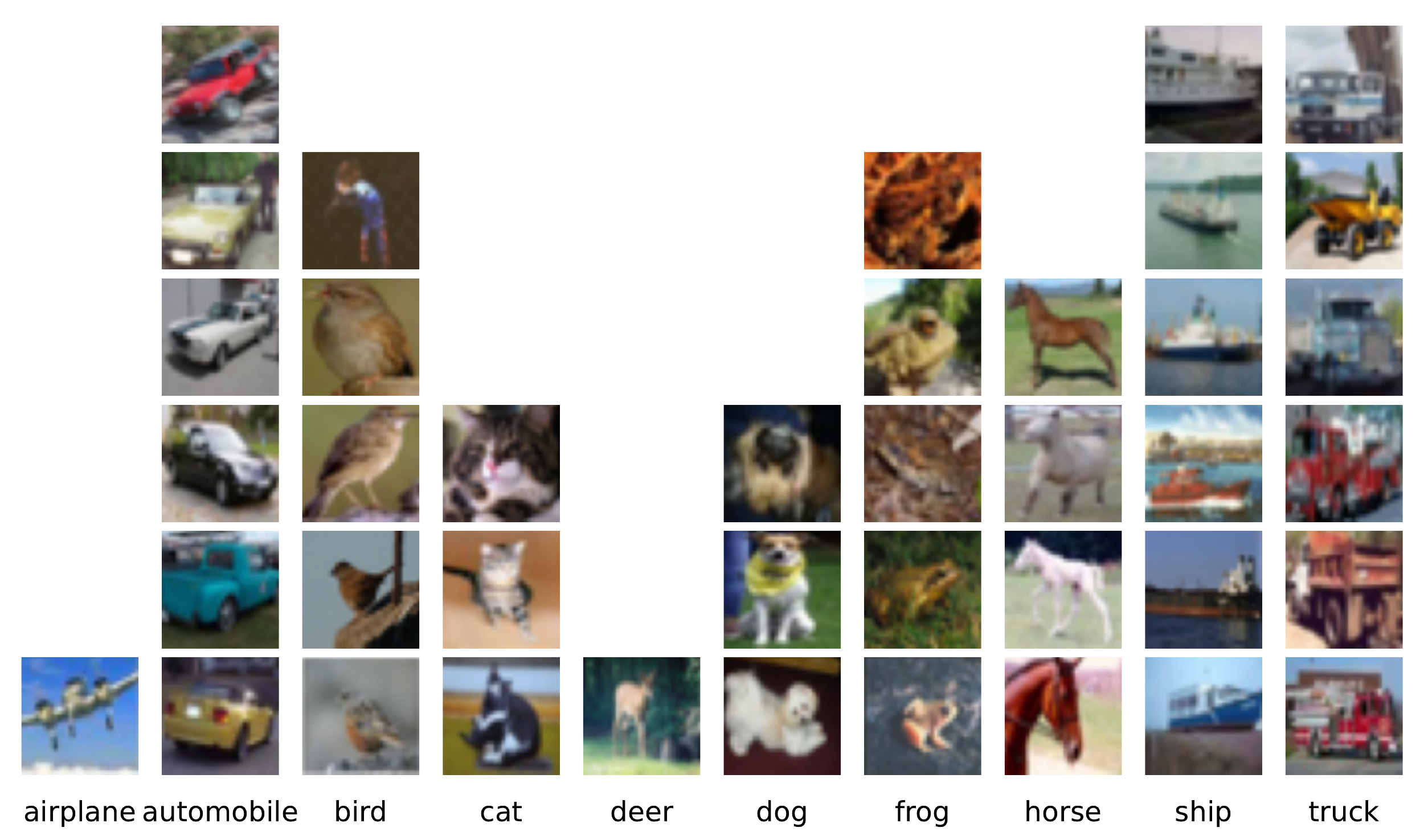}
    \label{fig:cifar-visualization-random}
    \caption{Random Selection: 40 Samples}
\end{subfigure} \\
\begin{subfigure}{1.0\linewidth}
    \centering
    \includegraphics[width=0.6\linewidth]{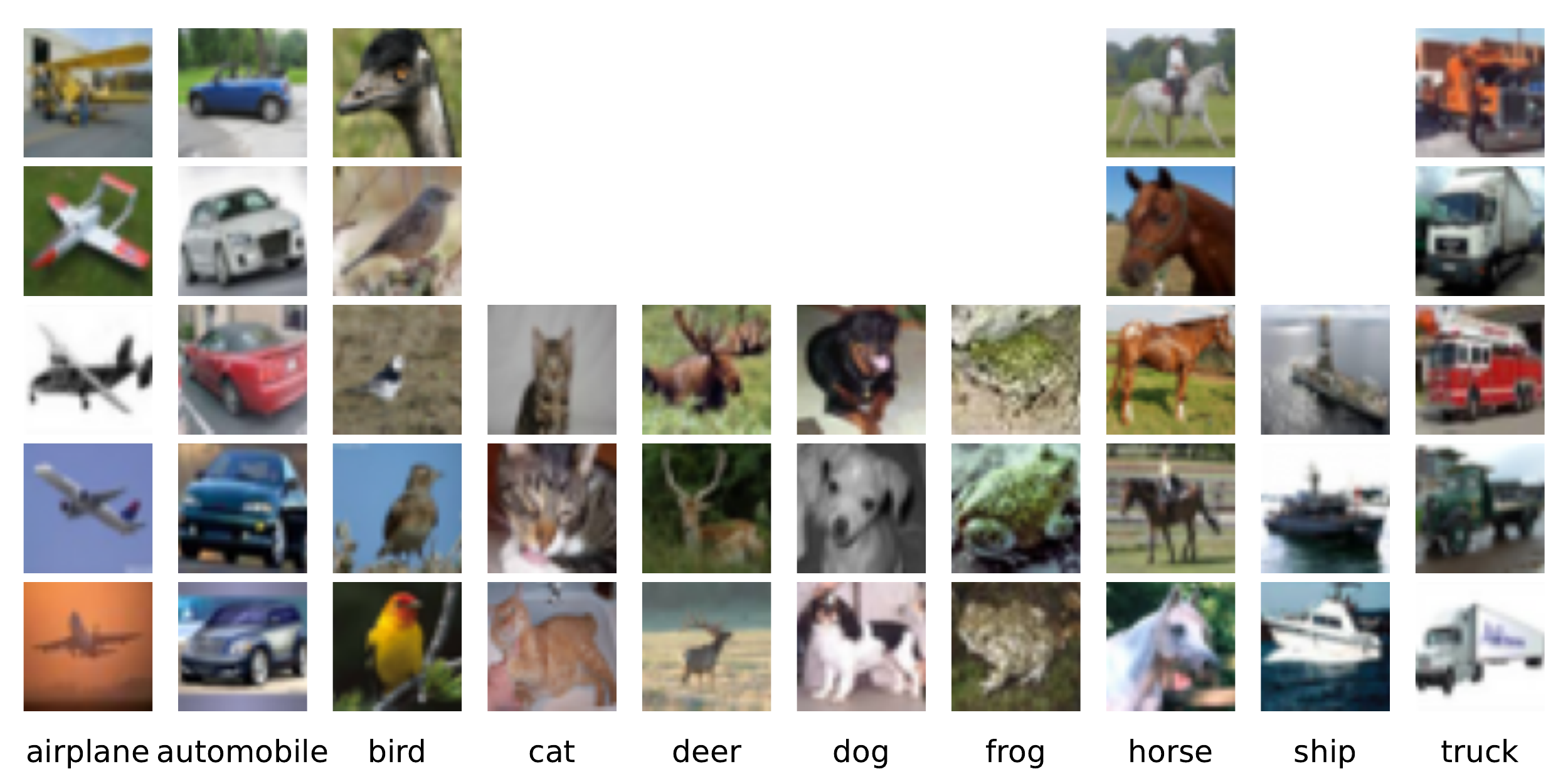}
    \label{fig:cifar-visualization-usl-top40}
    \caption{Ours (USL): Top-40 Selection}
\end{subfigure} \\
\begin{subfigure}{1.0\linewidth}
    \centering
    \includegraphics[width=0.6\linewidth]{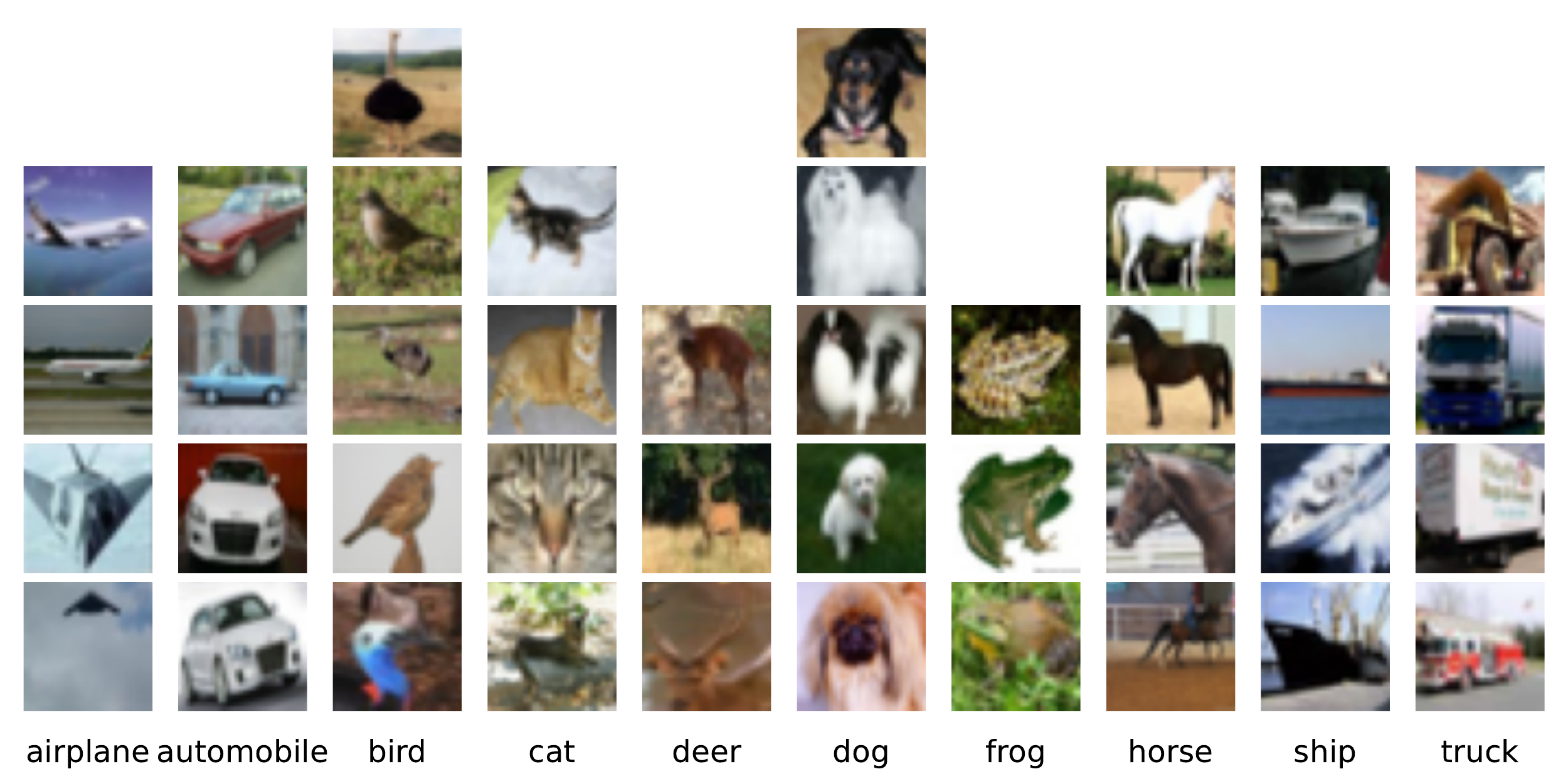}
    \label{fig:cifar-visualization-usl-t-top40}
    \caption{Ours (USL-T): Top-40 Selection}
\end{subfigure} \\
\begin{subfigure}{1.0\linewidth}
    \centering
    \includegraphics[width=0.6\linewidth]{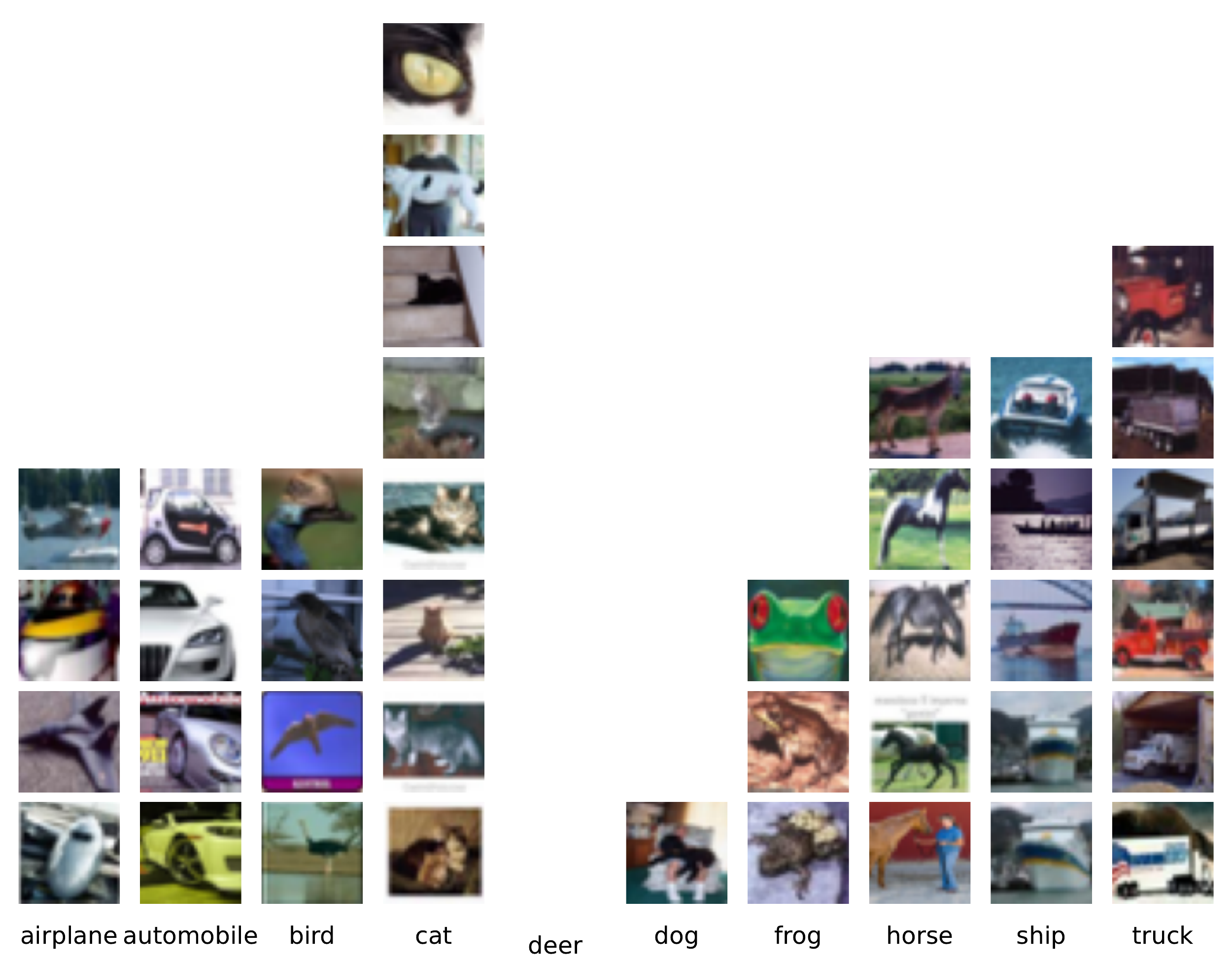}
    \label{fig:cifar-visualization-least40}
    \caption{Ours (USL): 40 Samples with Least Utility}
\end{subfigure}
\end{tabular}
\caption{Visualizations of selected samples in CIFAR-10: Our selections are mostly balanced and representative. In contrast, random selection is very imbalanced and the samples that we are least likely to select are almost always outliers.}
\label{fig:cifar-visualization}
\end{figure}
}

We visualize the top-40 and least-40 of our USL and USL-T selected samples in CIFAR-10, as in \fig{cifar-visualization}. For clarity, we put images into buckets according to their labels. Samples from random selection are highly imbalanced in terms of semantic class distribution and coverage. Our top selected samples from USL and USL-T are representative and diverse. The representativeness could be seen from that the objects are almost always appear without any occlusion or any truncation. In contrast, the 40 samples that we are least likely to select are mainly outliers that could mislead the classifier.

\newpage
\FloatBarrier
\figCIFARVisualization{h!}
\FloatBarrier

%%%%%%%%%%%%%%%%%%%%%%%%%%%%%%%%%%%%%%%%%%%%%%%%%%%%%%%%%%%%%%%%%%%%%%%%%%%%%%%%%%%%%%%%%%%%%%%%%%%%
\subsection{Pseudo-code for the Regularization Algorithm}
\label{appendix-regularization}
We summarize the regularization algorithm in pseudo-code in Alg. \ref{alg:regularization}. In Alg. \ref{alg:regularization}, we first obtain $\hat{\sV}^0$, the selection without regularization, and set the moving average regularizer $\hat{\text{Reg}}(V_i, 0)$ to 0 for every $V_i \in \sV$; then in each iteration, we update $\hat{\text{Reg}}(V_i, t)$ with moving average from a closeness measurement to other previously selected samples, where $t$ is the index of current iteration. We re-select samples according to regularized utility at the end of each iteration, with $\lambda$ being a balancing factor. In the end, the selection from the last iteration is returned.

\algRegularization{h!}

\subsection{Using Euclidean Distance or Cosine Similarity?}
\label{appendix-cosine-similarity}
Because the features of all instances are projected to a unit hypersphere with L2 normalization, theoretically, maximizing the cosine similarity between two nodes is equivalent to maximizing the inverse of Euclidean distance between two nodes: 
\begin{align}
\argmax_{i, j}(\lVert f(x_i) - f(x_j) \rVert_2)^{-1} &= \argmax_{i, j}(2 - 2\cos(f(x_i), f(x_j)))^{-1} \\
                                                     &= \argmax_{i, j}(\cos(f(x_i), f(x_j)))
\end{align}

However, empirically, using maximizing the inverse of Euclidean distance $1/d(\cdot)$ as the objective function performs better than maximizing the cosine similarity $\cos(x)$. The reason is that, when two nodes are very close to each others, $1/d(\cdot)$ is more sensitive to the change of its Euclidean distance, whereas $\cos(\cdot)$ tends to be saturated and insensitive to small changes. Therefore, the function $1/d(\cdot)$ has the desired property of non-saturating and can better focus on the distance difference with closest neighbors.

%%%%%%%%%%%%%%%%%%%%%%%%%%%%%%%%%%%%%%%%%%%%%%%%%%%%%%%%%%%%%%%%%%%%%%%%%%%%%%%%%%%%%%%%%%%%%%%%%%%%
\subsection{General-domain Multi-modal Models: our method on CLIP features}
\label{appendix-clip}
Although our method works well in both small and large scale datasets, there are still two interesting aspects that we would like to explore. \textbf{1)} In our approach, self-supervised models need to be re-trained for each new dataset, which is time-consuming and could potentially delay the schedule for data annotation in real-world industry. \textbf{2)} Unsupervised models do not model semantic information explicitly, which may lead to confusion that could potentially be mitigated (e.g. datasets with varying intra-class variance will take regions of different sizes and may be treated differently in an unexpected way).

To address these issues, we put our focus on a large pretrained model that encodes semantic information. % However, a large fully-annotated dataset, which is required to train a large supervised model, is very costly to obtain, and labelling the target dataset to train a supervised model defeats our purpose of requiring only a small amount of annotation on the target dataset. 
Fortunately, the availability of large-scale text-image pairs online makes it possible to train a large-scale model that encodes images in the general domain with semantic information. 
In this paper, we make use of publicly-available CLIP \cite{radford2021learning} models, a large-scale collection of models trained on Internet-crawled data with a wide general domain and use CLIP's image model as feature extractor.

Using models trained on multi-modal datasets resolves the above issues. Even though CLIP is never trained on our target dataset, nor does the categories in its training set match the dataset we are using, using it to select does not degrade our performance of sample selection and labeling pipeline. This indicates that the effectiveness of our label selection does not necessarily depend on whether the same pretrained model is used in the downstream task. In addition, we observe that such substitution even helps with a slightly larger annotation budget, demonstrating the effectiveness of making use of semantic information. Since we only perform inference on the CLIP model, the whole sample selection process could complete in \textit{0.5 hours} on a commodity server using one GPU, indicating the possibility of our methods without delaying the schedule of human annotation or modifying the annotation pipeline and enables it to be used by industry on real-world dataset collection.

Note that although CLIP supports zero-shot inference by using text input (e.g. class names) to generate weights for its classifier, it is not always possible to define a class with names or even know all the classes beforehand. Since we only make use of the image part of the CLIP model, we do not make use of prior text information (e.g. class descriptions) that are sometimes available in the real world. We leave better integration of our methods and zero-shot multi-modal models to future work.

\def\figHyperparamCombined#1{
    \begin{figure*}[#1]
    \centering
    \centering
    \includegraphics[width=0.96\textwidth]{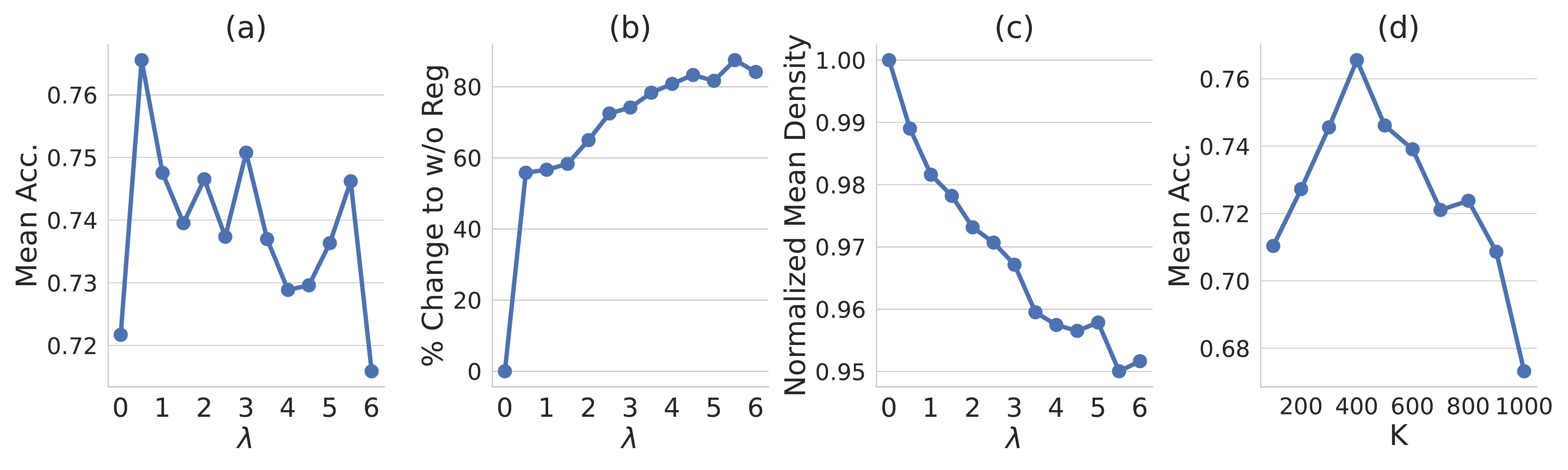}
    \caption{Effect of different hyperparameters, $\lambda$ (Fig. a,b,c) and $k$ (Fig. d) on CIFAR-10 with SimCLRv2-CLD. $\lambda$ balances representative and uniformity across the feature space. Larger $\lambda$ indicates stronger regularization that pushes more selections to be different but potentially selects less individually representative samples, or vice versa. Larger $k$ indicates that we are taking more neighbors into account when estimating the representativeness. Thanks to our stable formulation for density estimation, we found the optimal $k=400$ on CIFAR-10 also work \textit{consistently well} on CIFAR-100 and MedMNIST \cite{yang2021medmnist}, indicating the hyperparameter's insensitivity to number of classes, number of images in each class, and image domains.
    }
    \label{fig:hyperparam}
    \end{figure*}
}

\def\tabHyperparametersUSL#1{
    \begin{table}[#1]
    \centering
    \tablestyle{3pt}{1.0}
    \normalsize
    \tablefontsmall{
    \begin{tabular}{l|c|c|c||l|c|c}
    \shline
        & \multicolumn{3}{c||}{Small-scale Dataset} &    & \multicolumn{2}{c}{Large-scale Dataset} \\
    \hline
    Hyperparam & CIFAR-10 & CIFAR-100 & MedMNIST   & Hyperparam & ImageNet-100 & ImageNet  \\
    \shline
    $k$ in $k$NN & \multicolumn{3}{c||}{400}         & $k$ in $k$NN & \multicolumn{2}{c}{20} \\
    $m_\text{reg}$ & \multicolumn{3}{c||}{0.9}       & Horizon & \multicolumn{2}{c}{64} \\
    $\alpha$, $\lambda$ & \multicolumn{3}{c||}{0.5, 0.5($\leq$ 100 samples) /} & $\alpha$, $\lambda$ & \multicolumn{2}{c}{0.5, 1.5} \\
     & \multicolumn{3}{c||}{1.0, 1.0($>$ 100 samples)} &  & \multicolumn{2}{c}{} \\
    Iteration $l$ & \multicolumn{3}{c||}{10}      & & \multicolumn{2}{c}{} \\
    \shline
    \end{tabular}
    }
    \caption{A list of hyperparams used in our USL experiments. The hyperparameters are slightly different for small-scale and large-scale datasets due to the introduction of regularization horizon in selective labeling in large-scale datasets. Following \cite{sohn2020fixmatch}, we use different sets of hyperparameters for small-scale and large-scale datasets.}
    \label{tab:hyperparameters-USL}
    \end{table}
}

\subsection{Hyperparameter Analysis}
\label{appendix-hyperparam}
\figHyperparamCombined{t!}
\tabHyperparametersUSL{t!}

We focus on two hyperparameters in the analysis: $\lambda$, the weight for regularization, and $k$, the number of neighbors we use for $k$NN, in Fig. \ref{fig:hyperparam}. We use CIFAR-10 with SimCLRv2-CLD in a setting with a budget of 40 samples.

For hyperparam $\lambda$, we evaluated label selections with different $\lambda$ values used in regularization. In the experiments, we select $\lambda$, ranging from 0 to 6 in a 0.5 increment, where 0 indicates no regularization and larger $\lambda$ indicates a stronger regularization. We then evaluate the mean accuracy from 6 runs (using 2 runs per seed and 3 seeds per setting), the percent of samples that are different when compared to without regularization (i.e., $\lambda=0$), and mean density normalized w.r.t. without regularization. We observe that as $\lambda$ gets larger, we select more different samples compared to without regularization, which indicates stronger adjustment. This comes with higher accuracy as we have more uniformity. As a trade-off, we could not sample from area which has as high density as before because selecting samples from that area leads to selections that are close to each other, leading to a high penalty. Here, uniformity and representativeness show a trade-off and the optimal choice is to balance each other at $\lambda$ around 0.5. When $\lambda$ is much greater than 0.5, outlier samples that are as far away as possible from other selections are chosen without considering whether the selected samples are representative, which leads to much lower accuracy. 

For hyperparam $k$, we find that using a larger $k$ contributes to a better representation estimation by considering more neighbors. Thanks to our formulation that considers not only the $k^\text{th}$ sample for density estimation but the distance with all the $k$ nearest neighbors, we found that our algorithm's choice for $k$ is very generalizable: we found the optimal $k$ for CIFAR-10 to be 400, and found that $k=400$ also performs very well on CIFAR-100 and MedMNIST without \textit{any} tuning, which indicates our hyperparam's insensitivity in the number of classes, number of samples per class, and the dataset domain. Similarly, for larger scale datasets with higher image resolution and lower sample noise, we find that simply set $k=20$ leads to good performances on both 100 classes ImageNet and the full ImageNet with 1000 classes.

% A trade-off also exists: adding too much to $k$ makes the selection algorithm take too many false positives, i.e. samples from other semantic classes, into account. Empirically, for small-scale datasets, we found $k$ around 400 consistently performing well on CIFAR-10, CIFAR-100, and MedMNIST, despite that they have completely different number of instances per class and different difficulty to classify. For larger scale datasets with higher image resolution and lower sample noise, we find that simple set $k$ around $20$ leads to a good performance on both 100 classes ImageNet and the full ImageNet with 1000 classes.

%%%%%%%%%%%%%%%%%%%%%%%%%%%%%%%%%%%%%%%%%%%%%%%%%%%%%%%%%%%%%%%%%%%%%%%%%%%%%%%%%%%%%%%%%%%%%%%%%%%%

\subsection{Additional Discussions on Related Work}
\subsubsection{Related Work About Self-supervised Learning.}
Self-supervised Learning learns representations transferable to downstream tasks without annotations \cite{wu2018unsupervised, grill2020bootstrap}. \textit{Contrastive learning} \cite{wu2018unsupervised, he2020momentum, chen2020simple, wang2021unsupervised} learns representations that map similar samples or different augmentations of the same instance close  and dissimilar  instances apart.
\textit{Similarity-based} methods \cite{grill2020bootstrap} learn representations without negative pairs by predicting the embedding of a target network with an online network.
\textit{Feature learning with grouping} \cite{yang2010image, xie2016unsupervised, caron2018deep, zhuang2019local, caron2020unsupervised, wang2021unsupervised} respects the natural grouping of data by exploiting clusters in the latent representation. % and can be optimized together with instance-discrimination methods.
We study unlabeled data in a unsupervisedly learned feature space, due to its high quality and low feature dimensions.

We make use of the high-quality representations and dimensionality-reduction property in self-supervised learning to facilitate sample selection.

Using the representation learned with unsupervised learning as the feature space of selecting labels has two main advantages: 1) Without leveraging any labeled data, self-supervised learning could generate high-quality representations for many downstream tasks. 2) It relieves us from dealing with high-dimensional feature, due to relatively low dimension of output feature.

\subsubsection{Related Work About Our Deep Counterpart of $k$-Means Clustering in USL-T.}
\iffalse

\fi

In USL-T, we proposed a deep counterpart of $k$-Means clustering method that optimizes a unified global objective, which has an effect similar to performing $k$-Means clustering but trains the feature space and cluster assignment jointly. We would like to offer a comparison to main related work of our proposed method that also involves $k$-Means clustering variants or deep clustering designs to jointly learn features and cluster assignments.

Deep $k$-Means \cite{fard2020deep} proposed a differentiable metric on auto-encoder features to perform clustering. However, \cite{fard2020deep} only scales to small datasets such as MNIST, while our formulation scales to datasets with around a million images. In addition, while \cite{fard2020deep} requires a reconstruction term in the loss function to support clustering throughout training, our clustering loss, i.e. global loss, requires only one term that matches the soft and hard distribution. Note that although we also employ a local loss to kick-start the training process due to our confidence-based filtering function, the local loss could be turned off early in the training process without negative impacts on the clustering quality.

DeepCluster \cite{caron2018deep} also jointly learns features and cluster assignments with $k$-Means clustering. However, our work and \cite{caron2018deep} have different contributions: while our work adapts $k$-Means clustering to a unified loss formulation, \cite{caron2018deep} simply uses the traditional $k$-Means as a part of their algorithm to provide supervision for feature learning. In other words, while we directly back-propagation from our adapted $k$-Means algorithm as a global loss term, \cite{caron2018deep} uses traditional $k$-Means that does not supply gradients and employs another branch for back-propagation and learning purpose. In addition, \cite{caron2018deep} applies $k$-Means on features of all data, which means all feature needs to be stored prior to clustering, whereas we apply our loss formulation on the current minibatch, which adheres to popular deep learning methods that do not require storing all features from the dataset. USL-T, with end-to-end backprop to jointly solve for cluster assignments and model optimization, is much more \textit{scalable} and \textit{easy to implement}. 

Recent works \cite{chen2020unsupervised,caron2021emerging} on implementing clustering in a deep-learning framework incorporate neural networks that output a categorical distribution through a softmax operator at the end of the network. In addition, DINO \cite{caron2021emerging} also considers the potential collapses and proposes a carefully-designed loss function as mitigation. However, both methods mainly intend to learn a feature space/attention map used for downstream applications instead of acquiring a set of samples that are representative and diverse. Since the feature/attention maps are the goal of designing these methods, the \around60k clusters produced by DINO are extremely sparse and highly imbalanced. For ImageNet-1K, \around90\% clusters from a fully-trained DINO model are \textit{empty} (vs \around0 in USL-T). Therefore, the user has little control over the number of selections in DINO. Empirically, we observe that SSL models optimized on them perform much worse. Furthermore, in our unsupervised selective labeling setting, these methods require \textit{full retraining} when the downstream budget changes. In contrast, USL-T, which leverages self-supervised pretraining, could complete a selection with new budget constraint with substantially less compute.

Also recently, SCAN/NNM/RUC \cite{van2020scan,dang2021nearest,park2021improving} propose image clustering methods that intend to be evaluated with hungarian matching from image clusters to semantic classes. However, such methods are compared against semi-supervised learning methods \cite{van2020scan} instead of being proposed to be combined with semi-supervised learning methods. First of all, these methods make use of all labels on validation split to perform hungarian matching, which implicitly makes use of all the label information. In contrast, our USL/USL-T pipeline follows the standard assumption of semi-supervised learning that no labels, except the ones in the labeled dataset, are leveraged by the method to get the final classification. Furthermore, these methods generally do not generalize well to large datasets such as ImageNet\cite{ILSVRC15}, with \cite{dang2021nearest,park2021improving} working on smaller datasets and \cite{van2020scan} severely underperforms on ImageNet when a very limited amount (as low as 0.2\%) of data labels are available.

%%%%%%%%%%%%%%%%%%%%%%%%%%%%%%%%%%%%%%%%%%%%%%%%%%%%%%%%%%%%%%%%%%%%%%%%%%%%%%%%%%%%%%%%%%%%%%%%%%%%
\subsection{Overview on Unsupervised Representation Learning}
\label{appendix-representation-learning}
In self-supervised learning stage, we aim to learn a mapping function $f$ such that in the $f(x)$ feature space, the positive instance $x_i'$ is attracted to instance $x_i$, meanwhile, the negative instance $x_j$ (with $j\!\neq\! i$) is repelled, and we model $f$ by a convolutional neural network, mapping $x$ onto a $d$-dimensional hypersphere with $\normltwo$ normalization.
To make a fair comparison with previous arts \cite{cai2021exponential}, we use MoCo v2 \cite{chen2020improved} to learn representations on ImageNet with the instance-centric contrastive loss:
\begin{align}
    \begin{split}
    &C\left(f_i,f^+_i, f^-_{\neq i}\right) \!=\! \\
    &
    -\log \frac%
    {\exp(<f_i,f_i^+>\!/T)}
    {\exp(<f_i,f_i^+>\!/T)+\sum\limits_{j\neq i}\exp(<f_i,f_j^->\!/T)}\\[-6pt]
    \end{split}
    \label{eqn:instance-contrast}
\end{align}
where $T$ is a regulating temperature.
Minimizing it can be viewed as maximizing the mutual information (MI) lower bound between the features of the same instance \cite{hadsell2006dimensionality, oord2018representation}.
% $T$ is the temperature, regulating the strength of penalties on hard negatives.
For experiments on ImageNet, the MoCo model pre-trained for 800 epochs is used for initializing the SSL model, as in \cite{cai2021exponential}.

% The feature spaces of CIFAR-10 data we work on are extracted with CLD \cite{wang2021unsupervised}. The instance-group contrastive loss is added to the instance-centric contrastive loss in symmetrical terms over views $x_i$ and  $x_i'$:
The feature spaces of CIFAR-10 data we work on are extracted with CLD \cite{wang2021unsupervised}. The instance-group contrastive loss is added in symmetrical terms over views $x_i$ and  $x_i'$:

\begin{align}
    \begin{split}
    L(f;T_I,T_G,\lambda) \!=\! \sum_{i} ( & C(f_I(x_i),v_i,v_{\neq i};T_I)  \\
      + & C(f_I(x_i'),v_i,v_{\neq i};T_I)) \\
      + \lambda\sum_{i} ( & C(f_G(x_i'), M_{\Gamma(i)}, M_{\neq\Gamma(i)};T_G)  \\
      + &  C(f_G(x_i), M'_{\Gamma'(i)}, M'_{\Gamma'(i)};T_G) ) \\
    \end{split}
    \label{eqn:cld}
\end{align}
% {{\begin{align}
%     \begin{split}
%     &L(f;T_I,T_G,\lambda) \!=\!\sum_{i} \!
%       \underbrace{C(f_I(x_i),v_i,v_{\neq i};T_I) \!+\!
%       C(f_I(x_i'),v_i,v_{\neq i};T_I)}_{\text{instance-level discrimination}} \\
%     &\!+\!\lambda\sum_{i} \!
%       \underbrace{C(f_G(x_i'), M_{\Gamma(i)}, M_{\neq\Gamma(i)};T_G) \!+\!
%       C(f_G(x_i), M'_{\Gamma'(i)}, M'_{\Gamma'(i)};T_G)}_{\text{cross-level discrimination}}
%     \end{split}
%     \label{eqn:cld}
% \end{align}}}

Cross-level discrimination of Eqn.~\ref{eqn:cld} (second term) can be understood as minimizing the cross entropy between hard clustering assignment based on $f_G(x_i)$ and soft assignment predicted from $f_G(x_i')$ in a different view, where $f_G$ ($f_I$) is instance (group) branch, and $M_{\Gamma(i)}$ denotes the cluster centroid of instance $x_i$ with a cluster id $\Gamma(i)$ \cite{wang2021unsupervised}. Empirically, we found that CLD has great feature quality on CIFAR-10 and better respects the underlying semantic structure of data. To be consistent with original FixMatch settings, our semi-supervised learner on CIFAR-10 is trained from scratch, without using pretrained weights.

%%%%%%%%%%%%%%%%%%%%%%%%%%%%%%%%%%%%%%%%%%%%%%%%%%%%%%%%%%%%%%%%%%%%%%%%%%%%%%%%%%%%%%%%%%%%%%%%%%%%

\subsection{Discussions About Run Time}
\label{appendix-run-time}
CLD only takes about 4 hours to train on CIFAR-10 on a single GPU and sample selection with USL takes less than 10 minutes on CLD with one GPU. This takes significantly less GPU-time than FixMatch (120 GPU hours with 4 GPUs), which is, in turn, much less than the time for labelling the whole dataset of 50000 samples. On ImageNet, MoCo takes about 12 days with 8 GPUs to achieve 800 epochs \cite{he2020momentum}, our algorithm takes about an hour on one GPU to select samples for both 1\% and 0.2\% labels, and in the end, FixMatch takes another 20 hours on 4 GPUs to train. Although it sounds like we are using a lot of compute time just to train a self-supervised learning model for selecting what samples to annotate, the fact is that FixMatch requires a self-supervised pretrained checkpoint to work well when the number of labeled samples is low, as shown in \cite{cai2021exponential}, even \textit{without} our selection methods. The only compute overhead introduced is the sample selection process, which is \textit{negligible} when compared to the other two stages. In addition, shown in our experiments, CLIP, as a model trained on a general and diverse image-text dataset, could also be used to select samples with comparable and sometimes even better samples to label. This indicates that the self-supervised training stage is not required in our method for sample selection when a model that sufficiently covers the current domain is available.

%%%%%%%%%%%%%%%%%%%%%%%%%%%%%%%%%%%%%%%%%%%%%%%%%%%%%%%%%%%%%%%%%%%%%%%%%%%%%%%%%%%%%%%%%%%%%%%%%%%%

\subsection{Experiment Setup and Implementation Details}
\subsubsection{CIFAR-10/100.}
\label{appendix-setup-cifar}
For \textit{\textbf{FixMatch}} experiments, to maintain consistency with the original FixMatch \cite{sohn2020fixmatch}, we evaluate FixMatch trained on CIFAR-10 with $2^{20}$ steps in total. To illustrate the ability of our algorithm to select informative samples, we evaluate both approaches on an extremely low-label setting from 40 samples to 250 samples in total (4 shots to 25 shots per class on average).
Since the original FixMatch is evaluated with stratified sampling on CIFAR-10, we also retrain FixMatch with random sampling with the same number of samples in total as a fair comparison. 
Unless otherwise stated, we train FixMatch with a learning rate of 0.03, and weight decay $10^{-3}$ on 4 Nvidia RTX 2080 Ti GPUs with batch size 64 for labeled samples and with $2^{20}$ steps in total. All experiments are conducted with the same training and evaluation recipe for fair comparisons.

% \noindent\textbf{Setup For Transfer-Learning Based Approach SimCLRv2-CLD.}

For \textit{\textbf{SimCLRv2-CLD}}, we also evaluate our algorithm on two-stage SSL method SimCLRv2-CLD based on transfer learning \cite{chen2020big} by fine-tuning the linear layer of a ResNet-18 pretrained with self-supervised learning algorithm CLD \cite{wang2021unsupervised}. Specifically, we fine-tune the linear layer on a ResNet-18 trained with CLD \cite{wang2021unsupervised}.
Since it is easy for the network to overfit the few-shot labeled samples, we freeze the backbone and fine-tune only the linear layer. 
We use SGD with learning rate $0.01$, momentum $0.9$, and weight decay $10^{-4}$ for 5 epochs because longer training time will lead to over-fitting.

For \textit{\textbf{MixMatch}}, we train for 1024 epochs with 1024 steps per epoch, following the original recipe. For each of labeled and unlabeled dataset, we use a batch size 64. We use a learning rate $0.002$ with Adam optimizer. The results are evaluated with an weighted EMA module that has decay rate $0.999$ and are averaged over 20 last epochs in the test set. For \textit{\textbf{CoMatch}}, we train for 512 epochs with official code and the default recipe.

%\subsubsection{Experiment Setup for CIFAR-10 in the Introduction Section}
%The setup is exactly the same as SimCLRv2-CLD used in the experiment section. We skip the 2 of 5 random seeds as explained in the experiments section. For inter-fold standard deviation, we use the first 5 selections with instances in all 10 classes. For intra-fold standard deviation, we select samples with the first random seed. The average accuracy for inter-fold setting is 60.67\%, and the average accuracy for intra-fold setting is 60.96\%. Since the inter-fold accuracy is similar to intra-fold accuracy, our selection of intra-fold setting is not biased.

\subsubsection{ImageNet-100/1k.}
We evaluate our method on ImageNet \cite{ILSVRC15} with approximately 1 million images and 1000 classes and ImageNet100 \cite{van2020scan} with 100 classes from ImageNet.

We use different sets of hyperparameters in large-scale datasets, as described in Sec.~\ref{appendix-hyperparam}. For USL, we set a finite horizon in the large datasets to make evaluation feasible. Instead of using a momentum in regularization, we run one iteration without momentum for faster selection for both USL-MoCo and USL-CLIP. For USL-T, we freeze the backbone due to computational limitations in large-scale datasets. To maintain consistency with contrastive learning, we use $L^2$-normed linear layer as the last layer. We also initialize the last layer with features from random samples to greatly speed up convergence. As we find that providing only one label of the sample with top confidence in each cluster does not effectively convey the grouping information in low-shot SSL, we instead query the sample with top density in each cluster and annotate the 20 samples with max density using the label of the requested sample as the pseudo-label. We reduce the iterations in downstream for fair comparison. Similar to \cite{caron2018deep}, we re-initialize centroids of tail or empty clusters to the perturbed centroid of the head cluster. Since this creates centroid competitions that reduce confidence value of the head cluster, we do not make use of confidence value and calculate global loss on all samples by default.

For \textit{\textbf{SimCLRv2}} experiments, we fine-tune the released SimCLRv2 checkpoint on baseline selections and our selections. Due to differences in codebases, our reproduced accuracy differs from the one reported on SimCLRv2 paper pretrained and fine-tuned on Cloud TPUs. Therefore, we report our reproduced baseline which is fine-tuned on stratified selection for fair comparison on the effectiveness on sample selection with our method with SimCLRv2. Similar to other ImageNet experiments, we use 1\% and 0.2\% labeled data. The labeled data selection is the same for SimCLRv2 as for our experiments in FixMatch. To keep the recipe as close to the original implementation as possible, we use ResNet-50\cite{he2016deep} with LARS \cite{you2017large} optimizer with learning rate 0.16 and use globally synced batch normalization \cite{ioffe2015batch}. While \cite{chen2020big} employs a batch size of 1024, we found that under the same number of training epochs, setting batch size to 512 leads to better optimization outcomes on ImageNet-1k in our codebase. This is potentially due to more iterations with the same number of training epochs. Therefore, we set batch sizes to 512 on ImageNet-1k. In addition, to reduce the memory footprint, we use mixed precision training, which has no significant impacts in training accuracy in our observation. We use 60 epochs for 1\% task, following \cite{chen2020big}. We use 240 epochs for 0.2\% task without learning rate decay for all selection methods, since we find this gives better results.

% For each of these settings, we use either a MoCo-pretrained model with EMAN \cite{cai2021exponential} or a CLIP ViT/16 model \cite{dosovitskiy2020image} to select samples to label. 
% In SSL setting, we use FixMatch with MoCo EMAN parameters as initialization and the same setting as in \cite{cai2021exponential} besides the selection of labeled data unless stated.
% In SL setting, we use a batch size of 256 and train a ResNet-50 for 1000 epochs.

For \textit{\textbf{FixMatch}} experiments, we use either a MoCo-pretrained model with Exponential Moving Average Normalization (EMAN) \cite{cai2021exponential} or a CLIP ViT/16 model \cite{dosovitskiy2020image} to select samples to annotate. For ImageNet 1\%, we run $K$-Means clustering with 12900 clusters, which is slightly more than 12820 samples we are selecting, because we observe that there will sometimes be empty clusters. To maintain consistency with prior works, we use the same setting as in \cite{cai2021exponential} besides the selection of input labeled data, unless otherwise stated. Specifically we use a learning rate of 0.03 with weight decay $10^{-4}$ and train a ResNet-50 for 50 epochs with a MoCo \cite{he2020momentum} model as pretrained model. We perform learning rate warmup for 5 epochs and decay the learning rate by 0.1 at 30 and 40 epochs. Note that we load MoCo model as the pretrained model for FixMatch for USL-CLIP for fair comparison so that the only difference between MoCo and CLIP setting is the sample selection.

%%%%%%%%%%%%%%%%%%%%%%%%%%%%%%%%%%%%%%%%%%%%%%%%%%%%%%%%%%%%%%%%%%%%%%%%%%%%%%%%%%%%%%%%%%%%%%%%%%%%
\subsection{Details About the Toolbox}
Currently, different SSL/AL/SSAL implementations use different formats to represent what samples to label, making selective labeling methods hard to benchmark. Therefore, to standardize the benchmark, we intend to release a toolbox that includes implementations of following methods:

\begin{itemize}
  \item Our selective labeling methods: USL and USL-T
  \item SSL methods that we experimented on, including SimCLRv2 \cite{chen2020big}, SimCLRv2-CLD \cite{chen2020big, wang2021unsupervised}, FixMatch \cite{sohn2020fixmatch}, CoMatch \cite{li2021comatch}, MixMatch \cite{berthelot2019mixmatch}, that are adapted with the unified dataset representation as illustrated below
  \item Several AL/SSAL methods that we use as baselines
\end{itemize}

USL, USL-T, SSL methods, and the AL/SSAL baselines in the toolbox are implemented with \textit{unified} data loaders that comes with \textit{standard} and simple file formats to indicate what samples are requested to be labeled and what samples are unlabeled. We provide out-of-the-box data loaders that use this unified file representation for datasets used in our experiments. In addition, the training recipe will be provided for the methods mentioned above to facilitate future research and fair comparisons.

%%%%%%%%%%%%%%%%%%%%%%%%%%%%%%%%%%%%%%%%%%%%%%%%%%%%%%%%%%%%%%%%%%%%%%%%%%%%%%%%%%%%%%%%%%%%%%%%%%%%